\documentclass[10pt,twocolumn,letterpaper]{article}
\usepackage{xcolor}
\usepackage{booktabs}
\usepackage{amsmath, amssymb,amsfonts}
\usepackage{amsthm}
\newtheorem{lemma}{Lemma}

\usepackage{float}
\usepackage{csquotes}
\usepackage{bbding}

\usepackage[pagenumbers]{cvpr} 

\usepackage{float}
 
\definecolor{cvprblue}{rgb}{0.21,0.49,0.74}
\usepackage[pagebackref,breaklinks,colorlinks,allcolors=cvprblue]{hyperref}

\def\paperID{****} 
\def\confName{CVPR}
\def\confYear{2026}

\title{CGL: Advancing Continual GUI Learning via Reinforcement Fine-Tuning}

\newcommand{\equalcontrib}{\textsuperscript{*}}
\newcommand{\corresp}{\textsuperscript{\Envelope}}

\author{
    Zhenquan Yao\textsuperscript{1\equalcontrib}, 
    Zitong Huang\textsuperscript{1\equalcontrib}, 
    Yihan Zeng\textsuperscript{2$\dagger$}, 
    Jianhua Han\textsuperscript{2}, 
    Hang Xu\textsuperscript{2}, 
    Chun-Mei Feng\textsuperscript{3\corresp \equalcontrib} \\
    Jianwei Ma\textsuperscript{1,4\corresp \equalcontrib}, 
    Wangmeng Zuo\textsuperscript{1\corresp \equalcontrib} \thanks{
        \textsuperscript{*} These authors contributed equally to this work. \quad
        \textsuperscript{\Envelope} Corresponding authors.
        \textsuperscript{$\dagger$} Project leader.
    } \\[6pt]
    \textsuperscript{1}Harbin Institute of Technology \quad
    \textsuperscript{2}Huawei Noah's Ark Lab \quad
    \textsuperscript{3}University College Dublin \quad \\ \textsuperscript{4}Peking University \\
    {\small \texttt{25s012039@stu.hit.edu.cn, zitonghuang99@gmail.com}}
}

\begin{document}
\maketitle
\begin{abstract}
Graphical User Interface (GUI) Agents, benefiting from recent advances in multimodal large language models (MLLM), have achieved significant development. However, due to the frequent updates of GUI applications, adapting to new tasks without forgetting old tasks in GUI continual learning remains an open problem.  
In this work, we reveal that while Supervised Fine-Tuning (SFT) facilitates fast adaptation, it often triggers knowledge overwriting, whereas Reinforcement Learning (RL) demonstrates an inherent resilience that shields prior interaction logic from erasure. Based on this insight,
we propose a \textbf{C}ontinual \textbf{G}UI \textbf{L}earning (CGL) framework that dynamically balances adaptation efficiency and skill retention by enhancing the synergy between SFT and RL. 
Specifically, we introduce an SFT proportion adjustment mechanism guided by policy entropy to dynamically control the weight allocation between the SFT and RL training phases.
To resolve explicit gradient interference, we further develop a specialized gradient surgery strategy. By projecting exploratory SFT gradients onto GRPO-based anchor gradients, our method explicitly clips the components of SFT gradients that conflict with GRPO. On top of that, we establish an AndroidControl-CL benchmark, which divides GUI applications into distinct task groups to effectively simulate and evaluate the performance of continual GUI learning. Experimental results demonstrate the effectiveness of our proposed CGL framework across continual learning scenarios. The benchmark, code, and model will be made publicly available.

\end{abstract}    
\section{Introduction}

\begin{figure}[htbp]
\centering
\includegraphics[width=1\linewidth]{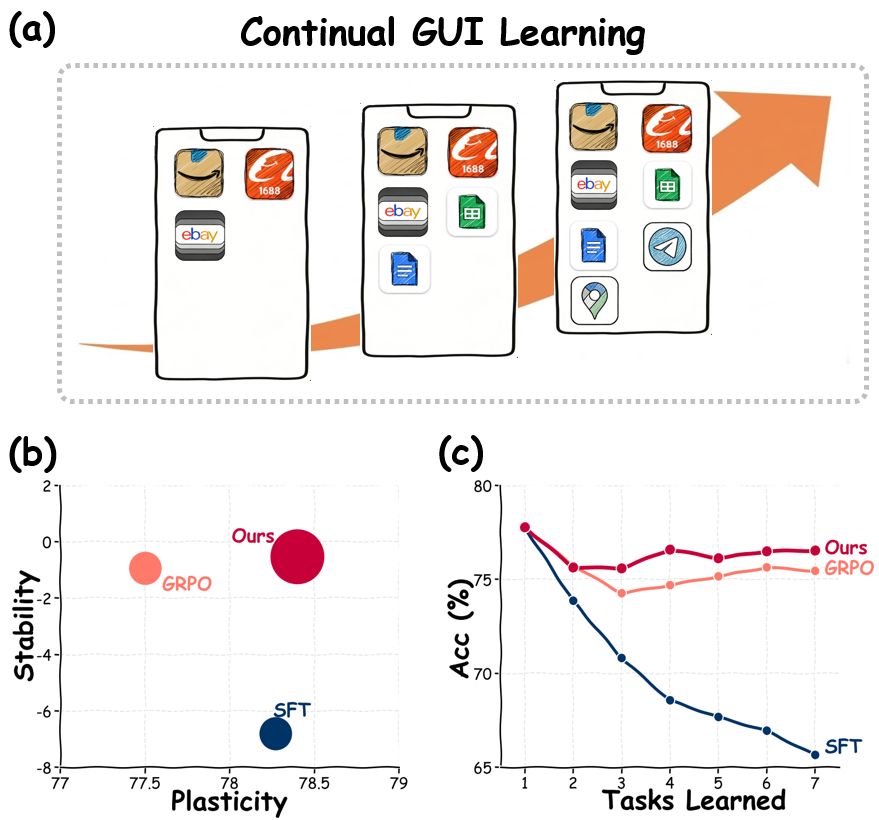}
\vspace{-4mm}
\caption{Overview of Continual GUI Learning (CGL).
(a) GUI environments evolve as new app categories emerge, causing static GUI agents to misalign with the real world. (b) Existing paradigms struggle to balance stability in legacy tasks and plasticity in novel task accuracy. (c) Our CGL framework achieves balanced adaptation and retention.}
\label{fig:setting}
\vspace{-18pt}
\end{figure}
\label{sec:intro}

Graphical User Interface (GUI) agents, powered by Multimodal Large Language Models (MLLMs)~\cite{zhu2025internvl3,bai2025qwen2,wang2024qwen2,li2024llava,chen2024expanding}, have demonstrated significant potential in automating complex interactions within software ecosystems~\cite{hong2024cogagent,luo2025gui,qin2025ui,lu2025ui,bai2024digirl,liu2025infiguiagent}. By interpreting visual semantics and executing hierarchical action sequences, these agents bridge the gap between high-level user intent and low-level interface execution. However, the prevailing training paradigm remains largely static, standing in stark contrast to the highly dynamic evolution of real-world GUI environments~\cite{lin2025showui,cheng2024seeclick,wu2024os-atlas,qin2025ui}.

Driven by frequent UI updates and functional iterations, GUI agents must possess Continual Learning (CL) capabilities, the ability to adapt to novel interfaces and tasks without compromising proficiency in previously mastered domains~\cite{zhu2024anytime,huang2024class2,wang2024learn,meng2025preserving}. However, CL in the GUI domain presents unique challenges distinct from traditional vision or natural language processing~\cite{zhu2024anytime,huang2024class2,wang2024learn}. GUI tasks involve exceptionally long-horizon action dependencies, where the integrity of an entire sequence hinges on precise intermediate executions. As shown in Fig.\ref{fig:setting}, exposing an agent to a sequence of evolving applications often results in a precipitous decline in performance on ancestral tasks. Even minor perturbations, such as a subtle shift in login layout or menu hierarchy, can render established interaction policies entirely obsolete.

Existing methods typically employ either Supervised Fine-Tuning (SFT) or Reinforcement Learning (RL), \ie, GRPO~\cite{guo2025deepseek}, to adapt agents to novel tasks~\cite{luo2025gui, qin2025ui,lu2025ui,hong2024cogagent,bai2024digirl,liu2025infiguiagent}. However, our investigation reveals a divergence between these two optimization paths. 
As shown in Fig.~\ref{fig:llava_adation}, the aggressive gradient updates in SFT forcibly pull the model parameters toward the manifold of the new task, thereby distorting the structural integrity of previously acquired knowledge. 
In contrast, GRPO~\cite{guo2025deepseek} exhibits a distinct inherent resilience. 
It preserves broader behavioral diversity, allowing the agent to optimize for rewards without entirely erasing the underlying interaction logic. 
Nevertheless, this stability is attained at the expense of high sample complexity, resulting in protracted adaptation speeds in unfamiliar environments that fail to meet real-world efficiency requirements.

Motivated by these insights, we propose the \textit{Continual GUI Learning (CGL)} framework, which establishes a synergistic mechanism encompassing conflict quantification and gradient rectification to reconcile the trade-off between adaptation efficiency and skill retention.
Specifically, we first introduce \textit{Error-Aware Routing} to dynamically trigger supervised corrective updates when reinforcement exploration fails. To synchronize knowledge acquisition with retention, we develop \textit{Entropy-Regulated Tuning}, which modulates the weight of the supervised objective based on real-time policy uncertainty. Furthermore, we implement \textit{Gradient Surgery} to project exploratory SFT gradients onto a conflict-free subspace defined by GRPO-based anchor gradients, effectively resolving parameter-level interference between the two objectives.
This projection process enables selective filtering of update directions, preserving only constructive components that are either orthogonal or aligned with established functional logic.
Consequently, CGL empowers the agent with superior exploration efficiency and evolutionary robustness without encroaching upon the \enquote{logical redline} of previously acquired skills.

To rigorously benchmark GUI-CL, we establish AndroidControl-CL, a dataset that partitions GUI tasks into sequential groups to simulate realistic software versioning. Our contributions are summarized as follows:

\begin{itemize}

\item We reveal that while SFT triggers knowledge overwriting, RL demonstrates inherent resilience in preserving prior GUI interaction logic.

\item We propose the CGL framework with entropy-guided SFT balancing and gradient surgery to achieve an effective trade-off between stability and plasticity.

\item We introduce AndroidControl-CL, providing a standardized platform for evaluating agent evolution under realistic distribution shifts.

\item Extensive experiments demonstrate that CGL significantly outperforms state-of-the-art baselines in both cross-domain adaptation speed and the mitigation of catastrophic forgetting.

\end{itemize}

\section{Related Work}

\subsection{Continual Learning}
Traditional continual learning methods are mainly categorized into three types \cite{liu2025continual}: Regularization-based methods \cite{zenke2017continual, li2017learning, aljundi2018memory} constrain critical parameter updates via importance quantification, such as adding regularization terms to protect old-task parameters; they are lightweight but face challenges in importance measurement and loss weight balancing. Dynamic Architecture methods \cite{yoon2017lifelong, rusu2016progressive} freeze old parameters and add task-specific sub-networks, which theoretically avoid forgetting but lead to linear model size expansion with increasing tasks. Rehearsal methods \cite{liu2023augmented, ji2023memorizing} use memory buffers to mix old and new task data during training, and their key challenge lies in balancing buffer size. For Vision-Language Models (VLMs), continual learning faces unique challenges like maintaining cross-modal alignment. Core solutions include Multi-Modal Replay (MMRE) \cite{zhang2023vqacl, chen2023continual, lei2023symbolic, marouf2025no,huang2024learning}: it adopts explicit or implicit replay, where explicit replay stores old image-text pairs in a buffer and implicit replay generates pseudo pairs, with limitations of storage consumption and dependence on pseudo-data quality respectively. Cross-Modal Regularization (CREG) \cite{yu2024select, zhou2025learning, huang2024class} adds constraint terms to retain old cross-modal associations, and its key issue is balancing constraint intensity. Parameter-Efficient Adaptation (PEA) \cite{li2024atlas, hu2022lora, jacobs1991adaptive,dong2024mr} uses lightweight modules without modifying core parameters to preserve zero-shot capability, but may struggle with complex tasks due to limited adjustable parameters. Recent work \cite{lai2025reinforcement} proposes a rollout-based Instance Filtering method (RIF-RFT) to enhance the stability and efficiency of continual post-training in VLMs.

\begin{figure}[t]
\centering
\vspace{-3pt}
\includegraphics[width=1.05\linewidth]{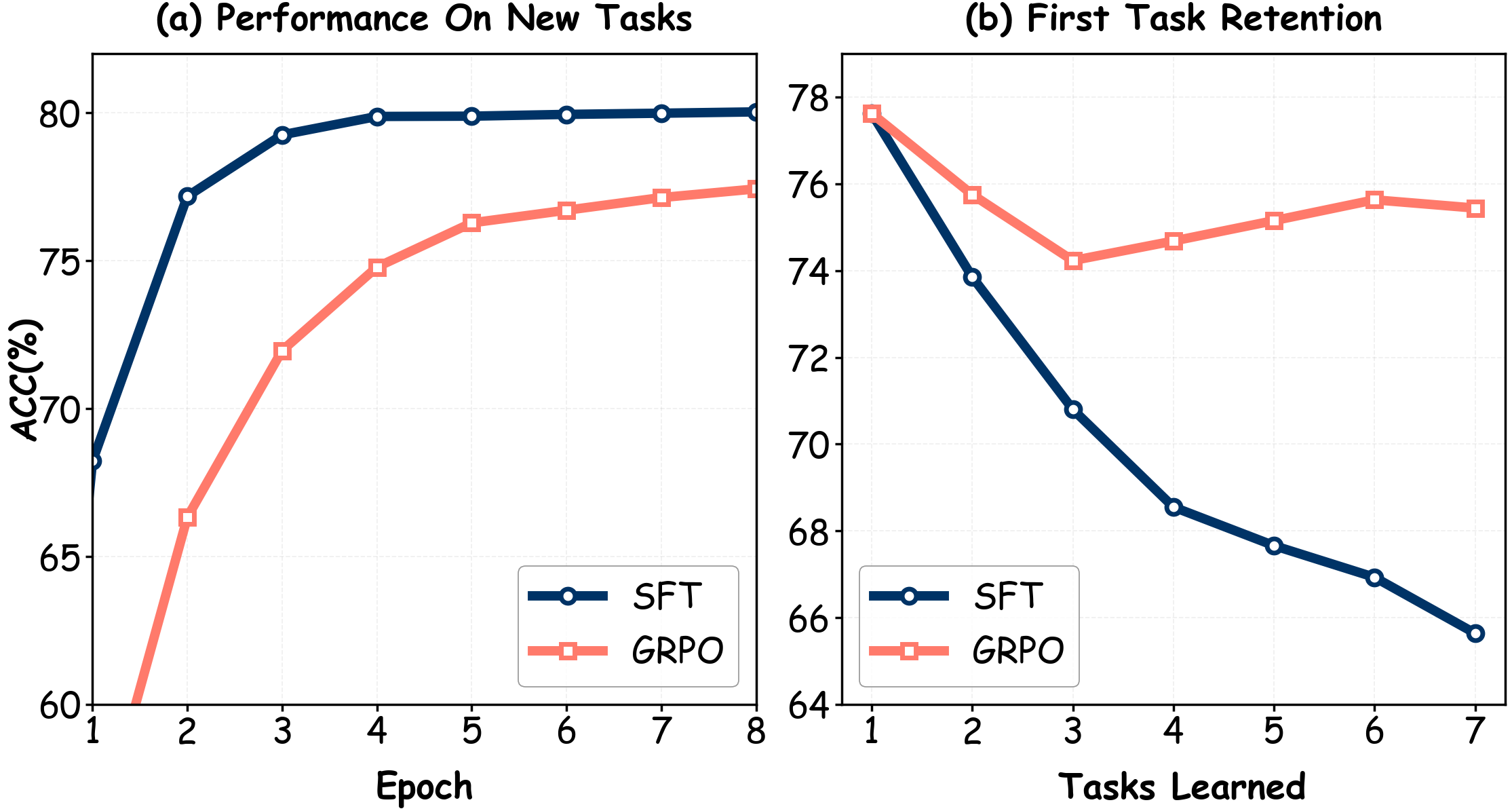}
\vspace{-4mm}
\caption{Preliminary comparison of SFT and GRPO on the LLaVA-OneVision-0.5B model. (a) Adaptation performance on a newly introduced task: SFT demonstrates fast plasticity, whereas GRPO exhibits slower adaptation. (b) Forgetting of the initial task after sequential training: SFT undergoes substantial degradation, while GRPO maintains higher retention.}
\label{fig:llava_adation}
\vspace{-13.5pt}
\end{figure}

\subsection{GUI Agents for Interactive Control}
GUI Agents automate human-like interactions with GUI systems by perceiving UI elements and executing target actions, with existing works primarily falling into two technical paradigms. SFT serves as a dominant data-driven paradigm, where specialized multimodal models are fine-tuned on large-scale labeled GUI corpora ~\cite{lin2025showui,cheng2024seeclick,wu2024os-atlas} to map UI states to actions, as exemplified by UI-TARS ~\cite{qin2025ui} which achieves strong benchmark performance but faces generalization limitations. RL-based GUI Agents, inspired by DeepSeek-R1~\cite{guo2025deepseek}, have gained traction: UI-R1\cite{lu2025ui} and GUI-G1~\cite{zhou2025gui} explores R1-zero-like visual grounding, while GUI-G2~\cite{tang2025gui} introduces Gaussian Reward Modeling and CRAFT-GUI~\cite{nong2025craft} proposes a curriculum-reinforced approach to enhance training stability and accuracy. This work \cite{xu2025exploration} dynamically schedules policy entropy to first increase and then decrease during training, enhancing the robustness of GUI agents against environmental noise. Notably, existing research focuses on static task scenarios without addressing the practical need for continuous adaptation to interface and task changes, while our method specifically targets to fill this gap.
Meanwhile, existing approaches merely rely on either SFT or RL for optimization, overlooking both their respective strengths and the potential of integrating~\cite{liu2026automated}. To the best of our knowledge, our work presents the first analysis of how SFT, RL, and their integration influence model capabilities within the dynamic setting of continual learning for GUI agent.
\section{Task Definition}

We formulate Continual GUI Learning (CGL) as a sequence of tasks where a multimodal agent must sequentially master interaction skills across disjoint application domains while preserving previously acquired capabilities. Formally, let $\mathcal{T} = \{T_1, T_2, \dots, T_N\}$ denote a sequence of $N$ tasks arriving chronologically. Each task $T_k$ comprises a set of mobile applications $\mathcal{A}_k = \{a_{k,1}, a_{k,2}, \dots, a_{k,M_k}\}$, where applications across distinct tasks are strictly non-overlapping: $\mathcal{A}_i \cap \mathcal{A}_j = \emptyset$ for all $i \neq j$. This ensures that each task introduces a novel application domain unseen in previous tasks.

Within each application $a_{k,m} \in \mathcal{A}_k$, the agent interacts through trajectories. A trajectory $\tau = (I, \{v_t, a_t\}_{t=1}^{T})$ consists of:
\begin{itemize}
    \item a natural language instruction $I$ specifying the target goal,
    \item a sequence of visual observations $v_t \in \mathbb{R}^{H \times W \times C}$ representing GUI screenshots at timestep $t$,
    \item a corresponding sequence of actions $a_t \in \mathcal{A}_{\text{GUI}}$ executed by the agent, where $\mathcal{A}_{\text{GUI}}$ denotes the unified action space (\eg coordinate-based taps, swipes, and text inputs).
\end{itemize}
For task $T_k$, we define a training set $\mathcal{D}_k^{\text{train}}$ and a test set $\mathcal{D}_k^{\text{test}}$, both composed of trajectories sampled from applications in $\mathcal{A}_k$.

\paragraph{Continual Learning Protocol}
The learning process follows a sequential paradigm where model parameters $\theta$ are updated iteratively across $\mathcal{T}$. At stage $k$, the agent initializes with weights $\theta_{k-1}$ learned from preceding tasks ($\theta_0$ denotes the initial pre-trained state), then trains exclusively on $\mathcal{D}_k^{\text{train}}$ without any access to historical training data $\{\mathcal{D}_1^{\text{train}}, \dots, \mathcal{D}_{k-1}^{\text{train}}\}$. The optimization aims to obtain updated parameters $\theta_k$ that minimize the current task loss:
\begin{equation}
\theta_k = \arg\min_{\theta} \mathcal{L}(\theta; \mathcal{D}_k^{\text{train}}),
\end{equation}
while simultaneously preserving proficiency on all previously encountered tasks. After training on $T_k$, the agent must generalize to trajectories from \emph{all} applications encountered so far, \ie it is evaluated on the cumulative test set $\bigcup_{i=1}^{k} \mathcal{D}_i^{\text{test}}$. This parameter inheritance mechanism under strict data isolation constraints embodies the core challenge of CGL: accumulating new interaction skills without catastrophically forgetting established GUI capabilities.

\paragraph{Metric.}  We adopt the following three metrics to evaluate the agent's capabilities.
\begin{itemize}
    \item \textbf{Step-wise Accuracy}: For each task, measures the per-step accuracy, \ie the proportion of correctly executed steps across all steps taken.
    \item \textbf{Average Step-wise Accuracy}: The mean value of step-wise accuracy across all evaluated tasks.
    \item \textbf{Trajectory-wise Accuracy}: For each task, measures the proportion of test trajectories in which every step is executed correctly, \ie the ratio of fully correct execution sequences to the total number of trajectories.
    \item \textbf{Average Trajectory-wise Accuracy}: The mean value of trajectory-wise accuracy across all evaluated tasks.
    \item \textbf{Forgetting Measure (FM)}: Quantifies the performance drop on previous tasks after learning new ones. Let $A_{n,k}$ denote the accuracy on task $T_k$ after training on task $T_n$ (where $n \ge k$):
    \begin{equation}
    FM = \frac{1}{N-1} \sum_{k=1}^{N-1} (A_{N,k} - A_{k,k})
    \end{equation}
\end{itemize}

\begin{figure*}[!htbp]
\centering
\vspace{-9mm}
\includegraphics[width=1\linewidth]{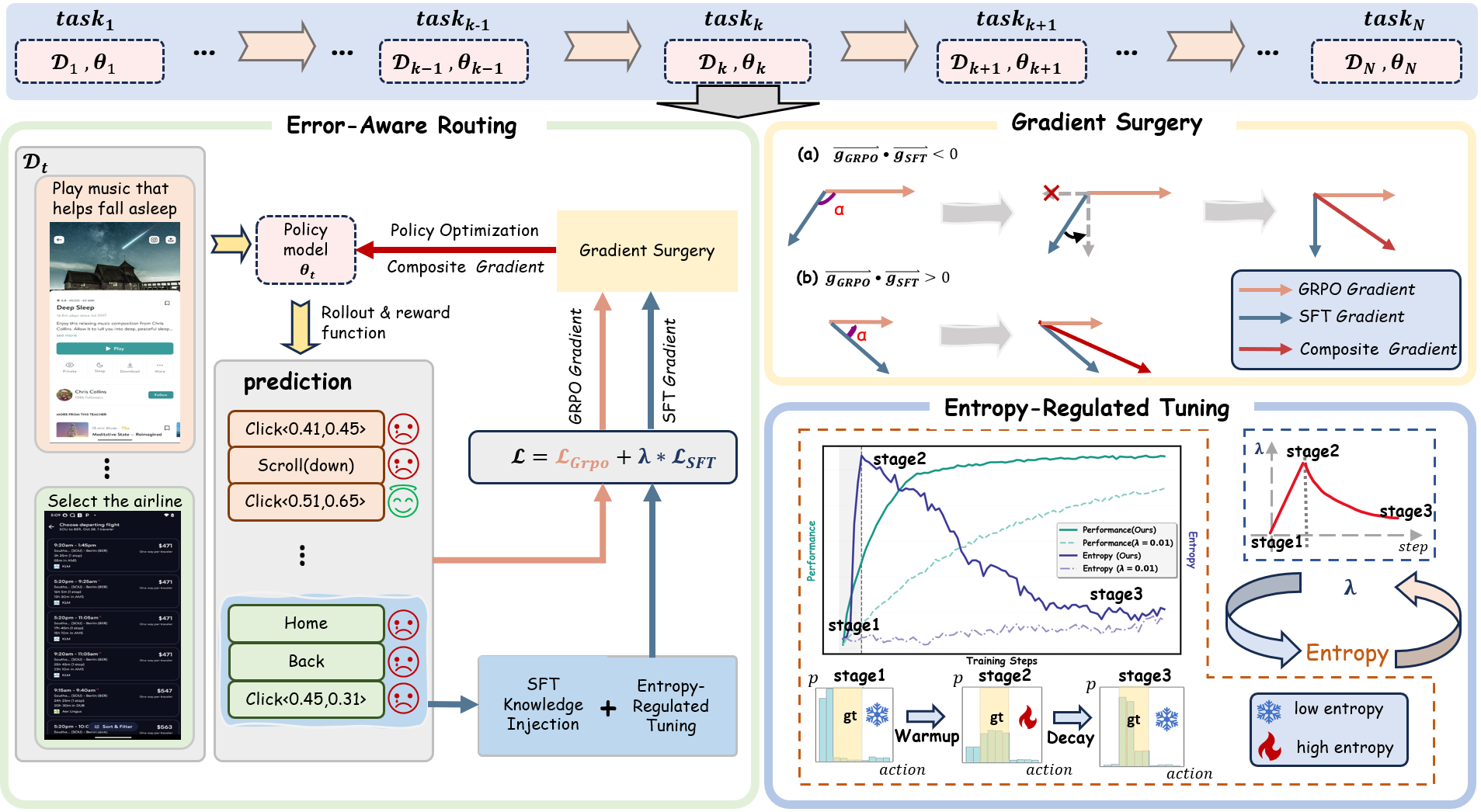}
\vspace{-4mm}
\caption{\textbf{Overview of the proposed CGL framework.} 
\textbf{(Left) Error-Aware Routing pipeline:}It dynamically routes data based on prediction feedback, where erroneous actions trigger SFT knowledge injection to facilitate correction. 
\textbf{(Top Right) Gradient Surgery:} Resolves directional conflicts between GRPO and SFT gradients through projection to maintain optimization stability. 
\textbf{(Bottom Right) Entropy-Regulated Tuning:} A dynamic strategy that manages the exploration-exploitation trade-off by adjusting the SFT weight $\lambda$. The bottom bar charts illustrate the model’s action probability distribution $p$, where the yellow shaded area denotes the ground truth ($gt$) action range.
}
\label{fig:method}
\vspace{-5pt}
\end{figure*}

\begin{table}[t]
    \vspace{-2mm}
    \centering
    \caption{Distribution of AndroidControl-CL. Abbreviations: SP (Shopping), PO (Productivity and Office), CO (Communication), TT (Travel and Transportation), ST (System Tools), ES (Education and Science), LE (Life and Entertainment).}
    \small
    \begin{tabular}{lccccccc}
    \toprule
    \textbf{Categories}       & \textbf{SP} & \textbf{PO} & \textbf{CO} & \textbf{TT} & \textbf{ST} & \textbf{ES} & \textbf{LE} \\
    \midrule
    \textbf{Apps}         & 15          & 15          & 8           & 10          & 13          & 11          & 17          \\
    \textbf{Trajectories}     & 1140        & 1240        & 617         & 686         & 771         & 559         & 1030        \\
    \bottomrule
    \end{tabular}
    \label{tab:androidcontrol-cl-distribution}
    \vspace{-5mm}
\end{table}

\section{AndroidControl-CL Benchmark}
Continual learning for GUI agents is essential in the face of rapidly emerging mobile applications and frequent software updates, yet existing datasets lack dedicated benchmarks tailored to dynamic GUI environments. 
To address this gap, we propose the \textbf{AndroidControl-CL} benchmark, an extension of AndroidControl~\cite{li2024effects} with additional annotations and structured task splits. 
This benchmark is specifically designed for the systematic evaluation of Continual GUI Learning (CGL).




\paragraph{Explicit, Multi-Source App Identity Annotation.}
The original dataset only includes sparse app identifiers inferred from occasional \texttt{open\_app} actions. AndroidControl-CL introduces a compact and robust attribution pipeline that integrates three complementary cues: (1) extracting explicit app identities from \texttt{open\_app} parameters; (2) identifying direct app references in natural-language instructions; and (3) performing screenshot-based app recognition when textual cues are insufficient. This yields clear, instance-level app identity labels.

\paragraph{Functional Super-Class Based Task Splits.}  
Building upon the per-trajectory app annotations, we structure AndroidControl-CL into a sequence of learning tasks based on app-level semantic categories. 
Specifically, we group apps into seven functional super-classes that reflect common mobile application domains: Shopping (SP), Productivity and Office (PO), Communication (CO), Travel and Transportation (TT), System Tools (ST), Education and Science (ES), and Life and Entertainment (LE). 
This design is motivated by the observation that apps within the same super-class tend to exhibit similar task logic, UI layouts, and interaction patterns, whereas those from different super-classes often present substantial distributional and functional gaps. 
In each task, we split the training and test sets at the trajectory level with an 8:2 ratio.
By organizing tasks along these super-class boundaries, we establish a realistic and challenging continual learning setting that evaluates an agent’s ability to incrementally acquire and generalize knowledge across distinct app domains. 

\paragraph{Balanced Task Distribution.}  
After task partitioning, we observed data imbalance across two dimensions: (1) imbalance numbers of apps across tasks, and (2) imbalance numbers of trajectory across apps. 
To ensure a fair evaluation setup, we perform targeted data filtering to balance both the number of apps per task and the number of trajectories per app.
The resulting balanced task distribution is shown in Tab.~\ref{tab:androidcontrol-cl-distribution}.

\paragraph{Fine-Grained Bounding Box Annotation for Click Actions.}  
In the original AndroidControl dataset, click actions are annotated only as single-point coordinates.
This fails to capture the realistic nature of GUI interactions, where a user typically targets a UI element (e.g., a button or icon) that occupies a spatial region rather than an infinitesimal point.
To better reflect this, we enhance the annotation by assigning each click action a bounding box that delineates the corresponding interactive UI element.
This is achieved through a designed three-stage pipeline: (1) automatic UI element parsing using OmniParser~\cite{lu2024omniparser} to extract candidate bounding boxes from screenshots; (2) matching the original annotated click coordinate to the nearest parsed UI element to obtain an initial bounding box assignment; and (3) manual verification and correction to resolve alignment errors or ambiguous cases. 
The resulting region-level annotations provide more semantically meaningful supervision, enabling GUI agents to learn grounded, element-aware action policies with higher fidelity.
\section{Method}
Following the established paradigm for GUI agents, our CGL framework is built upon Multimodal Large Language Models (MLLMs), such as LLaVA~\cite{li2024llava} or Qwen-VL~\cite{bai2025qwen2,wang2024qwen2}. In this setting, the agent functions as a policy network $\pi_\theta$ that operates over a sequential decision-making process. At each time step $t$, the model perceives a visual input consisting of the current screenshot $v_t$ and a global human instruction $I$ in natural language. The agent then processes these multimodal inputs to predict the next action $a_t$ in a text-based format, which is subsequently parsed into executable GUI commands. This formulation allows the agent to leverage the strong cross-modal capabilities of pre-trained MLLMs for interface navigation.
In preliminary experiments illustrated in Fig.~\ref{fig:llava_adation}, we identified complementary strengths between two training paradigms for GUI agents. \textbf{Supervised Fine-Tuning (SFT)} facilitates the efficient acquisition of new knowledge but remains susceptible to the catastrophic forgetting of established tasks. Conversely, \textbf{Reinforcement Learning (RL)} demonstrates superior retention of legacy knowledge but frequently fails to converge on new interaction patterns when rewards are sparse. 

To capitalize on these distinct advantages, we propose a joint training framework that integrates SFT with Group Relative Policy Optimization (GRPO) \cite{guo2025deepseek} illustrated in Fig.~\ref{fig:method}. 
Our approach is structured around three core modules designed to stabilize the training process and enhance the agent's adaptability. Specifically, \textbf{Error-Aware Routing} addresses the reward sparsity problem by dynamically invoking SFT demonstrations when RL exploration fails. \textbf{Entropy-Regulated Tuning} modulates the trade-off between exploration and exploitation by adjusting the weighting factor $\lambda$ based on policy uncertainty. Finally, \textbf{Gradient Surgery} resolves directional conflicts between the exploratory SFT updates and the stabilizing GRPO updates to ensure long-term knowledge retention.
The global optimization objective of our framework is formulated as a weighted combination of the two losses:
\begin{equation}
\mathcal{L} = \mathcal{L}_{\text{GRPO}} + \lambda(\mathcal{H}, \text{step}) \cdot \mathcal{L}_{\text{SFT}},
\end{equation}
where $\lambda$ is a function determined by policy entropy $\mathcal{H}$ and training step.

\subsection{Error-Aware Routing}
To achieve anti-forgetting in continual GUI learning, we adopt GRPO~\cite{guo2025deepseek} as the RL algorithm, leveraging its group-based relative advantage calculation and constrained policy update. For each GUI query \( q \) (screenshot,task instruction,and historical actions), we sample \( G \) trajectories from the old policy \( \pi_{\theta_{\text{old}}} \), computingnormalized advantages for each trajectory as:$A_{i,t} = \frac{r_i - \mu_r}{\sigma_r + \epsilon}$, where \( \mu_r \) and \( \sigma_r \) are the mean and standard deviation of rewards across the trajectory group, respectively. The details of the reward function are presented in Sup.\ref{reward_ui}. GRPO optimizes a constrained objective to stabilize policy updates:

\begin{equation} \label{eq:grpo}
\begin{split}
\mathcal{L}_{\text{GRPO}}(\theta) &=- \mathbb{E}\left[ q \sim P(Q), \{o_i\}_{i=1}^G \sim \pi_{\theta_{\text{old}}}(O|q) \right] \\
&\quad \left[ \frac{1}{G} \sum_{i=1}^G \frac{1}{|o_i|}\sum_{t=1}^{|o_i|} \min\bigg( \frac{\pi_\theta(o_{i,t}|q,o_{i,<t})}{\pi_{\theta_{\text{old}}}(o_{i,t}|q,o_{i,<t})} A_{i,t}, \right. \\
&\quad \left. \text{clip}\bigg( \frac{\pi_\theta(o_{i,t}|q,o_{i,<t})}{\pi_{\theta_{\text{old}}}(o_{i,t}|q,o_{i,<t})}, 1-\epsilon, 1+\epsilon \bigg) A_{i,t} \bigg) \right] \\
&\quad \left. - \beta \cdot \text{KL}(\pi_\theta \parallel \pi_{\text{ref}}) \right]
\end{split}
\end{equation} 

However, relying solely on GRPO for GUI agent training faces a critical limitation: when the policy has not yet acquired knowledge about a novel interaction pattern, all rollouts for a given instruction remain incorrect throughout execution. Since none of these trajectories achieve the full  reward, the relative advantage estimation degenerates—providing no discriminative signal to guide the policy toward correct behaviors. Consequently, the agent becomes trapped in unproductive exploration, severely hindering its ability to acquire new GUI skills.

To break this deadlock, we introduce an \textbf{Error-Aware Routing} mechanism that dynamically injects supervised guidance when RL signals become ineffective. Specifically, for each instruction $I$, we examine the rewards of all $K$ sampled rollouts $\{\tau_k\}_{k=1}^K$. If the maximum reward falls short of the ideal full score (i.e., $\max_k r(\tau_k) < r_{\text{max}}$), we interpret this as evidence that the policy cannot autonomously discover the correct solution. In such cases, we route the update to a supervised fine-tuning (SFT) step using the ground-truth demonstration $\tau^* = (I, \{v_t^*, a_t^*\}_{t=1}^{T^*})$, optimizing the likelihood:
\begin{equation} \label{eq:sft_total}
\mathcal{L}_{\text{SFT}} = -\frac{1}{|o^*|}\sum_{t=1}^{|o^*|} \log \pi_\theta\left(o_t^* \mid s, o_{<t}^*\right).
\end{equation}
Particularlly, for spatial actions, we employ an augmented SFT loss by sampling $G$ valid points within the target bounding box to enhance spatial generalization. This routing strategy ensures that the model only utilizes SFT to rectify "pathological" biases where GRPO exploration fails to find a successful path. 

\subsection{Entropy-Regulated Tuning}
\label{sec:entropy}

To manage the transition from exploration to exploitation across sequential tasks, we introduce a dynamic weighting factor $\lambda(\mathcal{H}, \text{step})$ to modulate the contribution of SFT.
The policy entropy, which measures the uncertainty of the agent's action distribution, is defined as $\mathcal{H}(\pi_{\theta}(\cdot|s)) = - \sum_{a \in \mathcal{A}} \pi_{\theta}(a|s) \log \pi_{\theta}(a|s)$,where $\mathcal{A}$ is the action space of the GUI agent. 

The mechanics of our collaborative tuning are rooted in the first-order dynamics of this entropy. Let $\mathbf{z}$ denote the policy logits such that $\pi_{\theta}(a|s) = \text{softmax}(z_a)$. An optimization step induces a logit update $\Delta z_a = \eta \nabla_{z_a} \mathcal{L}$, where $\eta$ is the learning rate. As derived in our theoretical analysis (see Sup.\ref{sec:theory}), the resulting change in entropy $\Delta \mathcal{H}$ can be approximated by the negative covariance between the current log-probabilities and the logit updates:
\begin{equation}
\label{eq:entropy_dynamics}
\Delta \mathcal{H} \approx - \text{Cov}_{a \sim \pi_{\theta}} \left( \log \pi_{\theta}(a|s), \Delta z_a \right).
\end{equation}
Eq.~\eqref{eq:entropy_dynamics} suggests that an update $\Delta z_a$ positively correlated with the current distribution reinforces confident actions and reduces entropy; conversely, an update that opposes the current bias injects entropy into the policy. Our strategy operates in two distinct phases based on this principle:

\textbf{Phase 1: Entropy Injection (Warmup).} During the initial $\text{step}_w$, we linearly increase $\lambda$ to its maximum. In this stage, the model typically suffers from a \textit{pathological bias} toward incorrect actions, where the probability of the ground-truth action $a^*$ vanishes ($\pi(a^*) \to 0$). The SFT update, $\Delta z_a^{\text{SFT}} \propto (\mathbb{I}[a=a^*] - \pi_{\theta}(a|s))$, assigns a large positive update to the low-probability target $a^*$ and negative updates to the high-probability erroneous actions. This creates a strong \textbf{negative covariance}, acting as an \textbf{entropy injector} ($\Delta \mathcal{H}_{\text{SFT}} > 0$). This process "heats up" the distribution, breaking local minima and forcing the agent to explore the ground-truth space.

\textbf{Phase 2: Entropy Decay (Convergence).} Once basic task competency is established, we trigger a decay phase where $\lambda$ follows an exponential function of $\mathcal{H}$:
\begin{equation} \label{eq:lambda_decay}
\lambda(\mathcal{H}) = (\lambda_{\max} - \lambda_{\min}) \min\left(1, ke^{\gamma \mathcal{H}}\right) + \lambda_{\min}.
\end{equation}
In this phase, GRPO updates ($\Delta z_a^{\text{GRPO}} \propto \pi_{\theta}(a|s) A(s,a)$) dominate the optimization. Due to the \textbf{Matthew Effect} (see Sup.\ref{Matthew}), GRPO reinforces actions that already possess high probabilities and positive advantages. This induces a \textbf{positive covariance} that drives entropy down ($\Delta \mathcal{H}_{\text{GRPO}} < 0$). By decaying $\lambda$ as $\mathcal{H}$ drops, we ensure that the SFT component does not interfere with GRPO's precise convergence, allowing the model to solidify its knowledge for long-term retention.

\subsection{Conditional Gradient Surgery for Conflict Resolution}
During collaborative training of Supervised Fine-Tuning (SFT) and GRPO (a reinforcement learning-based objective), gradient conflicts between the two losses can destabilize optimization and impair knowledge preservation. To address this, we introduce a \textit{conditional} gradient surgery strategy: the SFT gradient is modified \textit{only when} it conflicts with the GRPO gradient’s anti-forgetting direction; otherwise, it remains unchanged.

\paragraph{Conflict Detection Criterion}
We define a conflict as occurring when the cosine similarity between the two gradients is negative (i.e., the angle exceeds $90^\circ$):
\begin{equation} \label{eq:conflict_criterion}
\cos\alpha = 
\frac{
\nabla_\theta \mathcal{L}_{\text{SFT}} \cdot \nabla_\theta \mathcal{L}_{\text{GRPO}}
}{
\|\nabla_\theta \mathcal{L}_{\text{SFT}}\|_2 \, \|\nabla_\theta \mathcal{L}_{\text{GRPO}}\|_2
} < 0.
\end{equation}
If $\cos\alpha \geq 0$, the gradients are deemed compatible, and the original SFT gradient is used directly.

\paragraph{Orthogonal Projection under Conflict}
When a conflict is detected ($\cos\alpha < 0$), we surgically remove the component of the SFT gradient that opposes the GRPO direction. Specifically, we retain only the part of $\nabla_\theta \mathcal{L}_{\text{SFT}}$ orthogonal to $\nabla_\theta \mathcal{L}_{\text{GRPO}}$:
\begin{equation} \label{eq:projected_gradient}
\nabla_\theta \mathcal{L}_{\text{SFT}^*} = \nabla_\theta \mathcal{L}_{\text{SFT}} - \nabla_\parallel,
\end{equation}
where the conflicting parallel component is:
\begin{equation} \label{eq:parallel_component}
\nabla_\parallel = 
\left(
\frac{
\nabla_\theta \mathcal{L}_{\text{SFT}} \cdot \nabla_\theta \mathcal{L}_{\text{GRPO}}
}{
\|\nabla_\theta \mathcal{L}_{\text{GRPO}}\|_2^2
}
\right)
\nabla_\theta \mathcal{L}_{\text{GRPO}}.
\end{equation}
This projection eliminates the adversarial update direction while preserving all SFT information orthogonal to GRPO’s objective. The final update gradient for SFT is thus:
\[
\nabla_\theta \mathcal{L}_{\text{SFT}}^{\text{final}} = 
\begin{cases}
\nabla_\theta \mathcal{L}_{\text{SFT}^*}, & \text{if } \cos\alpha < 0 \quad \text{(conflict)}, \\
\nabla_\theta \mathcal{L}_{\text{SFT}}, & \text{otherwise} \quad \text{(no conflict)}.
\end{cases}
\]

\section{Experiment}

\begin{table*}[t]
    \centering
    \vspace{-9mm}
     \caption{Performance comparison of LLava-onsevision-0.5b under Task Order 1. All accuracy values are reported as percentages (\%).}
     \vspace{-4mm}
    \resizebox{\linewidth}{!}{
    \begin{tabular}{lccccccc|cccc}
    \toprule
    \textbf{Method} & \textbf{Shopping} & \textbf{Productivity} & \textbf{Communication} & \textbf{Travel} & \textbf{Tools} & \textbf{Education} & \textbf{Entertainment} & \textbf{Avg. Step-Acc.} & \textbf{Avg. Trajectory Acc.} & \textbf{Avg. FM} \\
    \midrule
    SFT                               &65.64          & 74.73          & 74.66          & 72.00          & 77.39          & 64.12          & 80.08          & 72.66           &  14.61 & -6.81 \\
    SFT+KL                                  & 74.16 &75.23 & 76.89 & 75.49   & 82.06   &   66.01   &  79.29  & 75.59 &   22.39    & -1.78\\
    SFT+Replay                              &70.88          & 76.43          & 77.23          & 74.19          & 78.21          & 64.37          & \textbf{80.63}          & 74.56           &  20.07 & -4.69 \\
    GRPO\cite{guo2025deepseek}                              & 75.44          & \underline{78.86}          & \underline{78.01}          & \underline{76.41}          & \underline{81.57}          & \underline{67.93}          & 77.80          & \underline{76.57}         & \underline{23.85} & -0.93 \\
    RIF-RFT\cite{lai2025reinforcement}                      & \textbf{76.53}          & 74.19          & 75.99          & 73.55          &81.40           & 65.91          & 75.34          & 74.40          &  22.17          &   \underline{-0.88}             \\
    Ours                                   & \underline{75.55}          & \textbf{81.32}          & \textbf{80.73}          & \textbf{76.88}          & \textbf{83.41}          & \textbf{68.90}          & \underline{78.13}          & \textbf{77.84}      & \textbf{24.77}   &   \textbf{-0.52}\\
    \midrule
    SFT-Joint-Training                & 77.83          & 84.46          & 83.24          & 78.41          & 84.76          & 69.73          & 79.58          & 79.69            & 28.63     & -                \\
    Zero-Shot                & -          & -          & -          & -          & -          & -          & -          & -            & -     & -                \\
    \bottomrule
    \end{tabular}}
    \label{tab:method-task-performance-llava-order1}
\end{table*}

\begin{table*}[t]
    \centering
    \vspace{-2mm}
    \caption{Performance comparison of QwenVL2.5-3b under Task Order 1. All accuracy values are reported as percentages (\%). }
    \vspace{-4mm}
    \resizebox{\linewidth}{!}{
    \begin{tabular}{lccccccc|cccc}
    \toprule
    \textbf{Method} & \textbf{Shopping} & \textbf{Productivity} & \textbf{Communication} & \textbf{Travel} & \textbf{Tools} & \textbf{Education} & \textbf{Entertainment} & \textbf{Avg. Step Acc.} & \textbf{Avg. Trajectory Acc.} & \textbf{Avg.FM} \\
    \midrule
    SFT                          & 70.54          & 79.52          & 78.90          & 76.41          & 83.24          & 69.24          & 80.42          & 76.90          & 23.53                & -5.73          \\
    SFT+KL                       & 79.30          & 83.85          & 82.32          & 79.84          & 85.83          & 74.08          & 80.71          & 80.84          & 34.69                & -1.01                \\
    SFT+Replay                   & 74.61          & 83.35          & 81.66          & 79.02          & 85.61          & 72.26          & \underline{82.08}          & 79.80          & 30.19                & -3.11                \\ 
    GRPO\cite{guo2025deepseek}                         & \underline{79.63}          & \underline{84.74}          & \textbf{83.43}          & \underline{81.34}          & 85.74          & \underline{74.08}          & 81.73          & \underline{81.53}          & \underline{36.78}           & -0.62           \\
    RIF-RFT\cite{lai2025reinforcement}                      & 79.53          & 83.35          & 80.77          & 80.25          & \underline{85.94}          & 73.00          & 80.25          & 80.44          & 32.91           & \underline{-0.58}                \\
    Ours                         & \textbf{80.32}          & \textbf{85.46} & 81.99          & \textbf{82.03} & \textbf{88.17} & \textbf{75.81} & \textbf{82.53} & \textbf{82.33} & \textbf{38.03} & \textbf{-0.02}                \\
    \midrule
    SFT-Joint-Training               & 81.73 & 86.74 & 85.75 & 82.24 & 88.17 & 77.00 & 82.76 & 83.48 & 41.66           & -                \\
    Zero-Shot                & 48.02          & 57.71          & 49.50          & 52.81          & 55.87          & 42.33          & 48.86          & 50.73            & 1.86     & -                \\
    \bottomrule
    \end{tabular}}
    \vspace{-6mm}
    \label{tab:method-task-performance-qwenvl-order1}
\end{table*}

\begin{table*}[t]
    \centering
    \vspace{-9mm}  
    \caption{Model performance comparison across different task orders.}
    \vspace{-4mm}  
    \resizebox{\linewidth}{!}{
    \begin{tabular}{lccccccccc}
    \toprule
    \textbf{Model} & \multicolumn{3}{c}{\textbf{Task Order 1}} & \multicolumn{3}{c}{\textbf{Task Order 2}} & \multicolumn{3}{c}{\textbf{Task Order 3}} \\
    \cmidrule(lr){2-4} \cmidrule(lr){5-7} \cmidrule(lr){8-10}
    & Step Acc. (\%) & Trajectory Acc. (\%) & FM & Step Acc. (\%) & Trajectory Acc. (\%) & FM & Step Acc. (\%) & Trajectory Acc. (\%) & FM \\
    \midrule
    SFT           & 76.90 & 23.53 & -5.73 & 77.21 & 23.46 & -5.33 & 76.87 & 23.69 & -6.03 \\
    SFT+KL         & 80.84 & 34.69 & -1.01  & 80.77  & 34.33 & -0.96 & 80.69 & 34.13 & -1.13 \\
    SFT+Replay    & 79.80 & 27.19 & -3.11 & 80.53 & 28.06 & -2.70 & 80.39 & 27.34 & -3.29 \\
    GRPO\cite{guo2025deepseek}          & \underline{81.53} & \underline{36.78} & -0.62 & \underline{81.75} & \underline{36.60} & \underline{-0.33} & \underline{81.44} & \underline{36.72} & \underline{-0.52} \\
    RIF-RFT\cite{lai2025reinforcement}       & 80.44 & 32.91 & \underline{-0.58} & 80.50 & 33.02 & -0.46 & 80.35 & 31.88 & -0.55 \\
    Ours          & \textbf{82.33} & \textbf{38.03} & \textbf{-0.02} & \textbf{82.61} & \textbf{38.66} & \textbf{+0.13} & \textbf{82.40} & \textbf{38.57} & \textbf{-0.06} \\
    \bottomrule
    \end{tabular}}
    \label{tab:cross-order-performance}
    \vspace{-3mm}  
\end{table*}


\begin{table}[t]  
    \centering
    \caption{Ablation study of modules.}
    \vspace{-4mm}
    \scriptsize  
    \renewcommand{\arraystretch}{0.8}  
    \setlength{\tabcolsep}{3pt}  
    \setlength{\cmidrulewidth}{0.6pt}  
    \begin{tabular}{lccccccccc}  
    \toprule
    \multicolumn{6}{c}{\textbf{Modules}} & \multicolumn{3}{c}{\textbf{Metrics}} \\
    \cmidrule(lr){1-6} \cmidrule(lr){7-9}
    \textbf{SFT} & \textbf{GRPO} & \textbf{KL} & \textbf{D-SFT} & \textbf{D-$\lambda$} & \textbf{G-Surg} & \textbf{Step-Acc} & \textbf{Traj-Acc} & \textbf{FM} \\
    \midrule
    $\surd$ &  &  &  &  &  & 76.90 & 23.53 & -5.73 \\
    $\surd$ &  & $\surd$ &  &  &  & 80.84 & 34.69 & -1.01 \\
    & $\surd$ & $\surd$ &  &  &  & 81.53 & 36.78 & -0.62 \\
    & $\surd$ &  &  &  &  & -- & -- & -- \\
    $\surd$ & $\surd$ & $\surd$ &  &  &  & 81.68 & 36.79 & -0.57 \\
    $\surd$ & $\surd$ & $\surd$ & $\surd$ &  &  & 81.90 & 37.12 & -0.33 \\
    $\surd$ & $\surd$ & $\surd$ & $\surd$ &  & $\surd$ & 82.10 & 37.57 & \underline{-0.05} \\
    $\surd$ & $\surd$ & $\surd$ & $\surd$ & $\surd$ &  & \underline{82.23} & \underline{37.62} & -0.07 \\
    $\surd$ & $\surd$ & $\surd$ & $\surd$ & $\surd$ & $\surd$ & \textbf{82.33} & \textbf{38.03} & \textbf{-0.02} \\
    \bottomrule
    \end{tabular}
    \vspace{-3mm}  
    \label{tab:ablation-study}
\end{table}

\begin{figure*}[htbp]
\centering
\vspace{-1mm}
\includegraphics[width=0.75\linewidth]{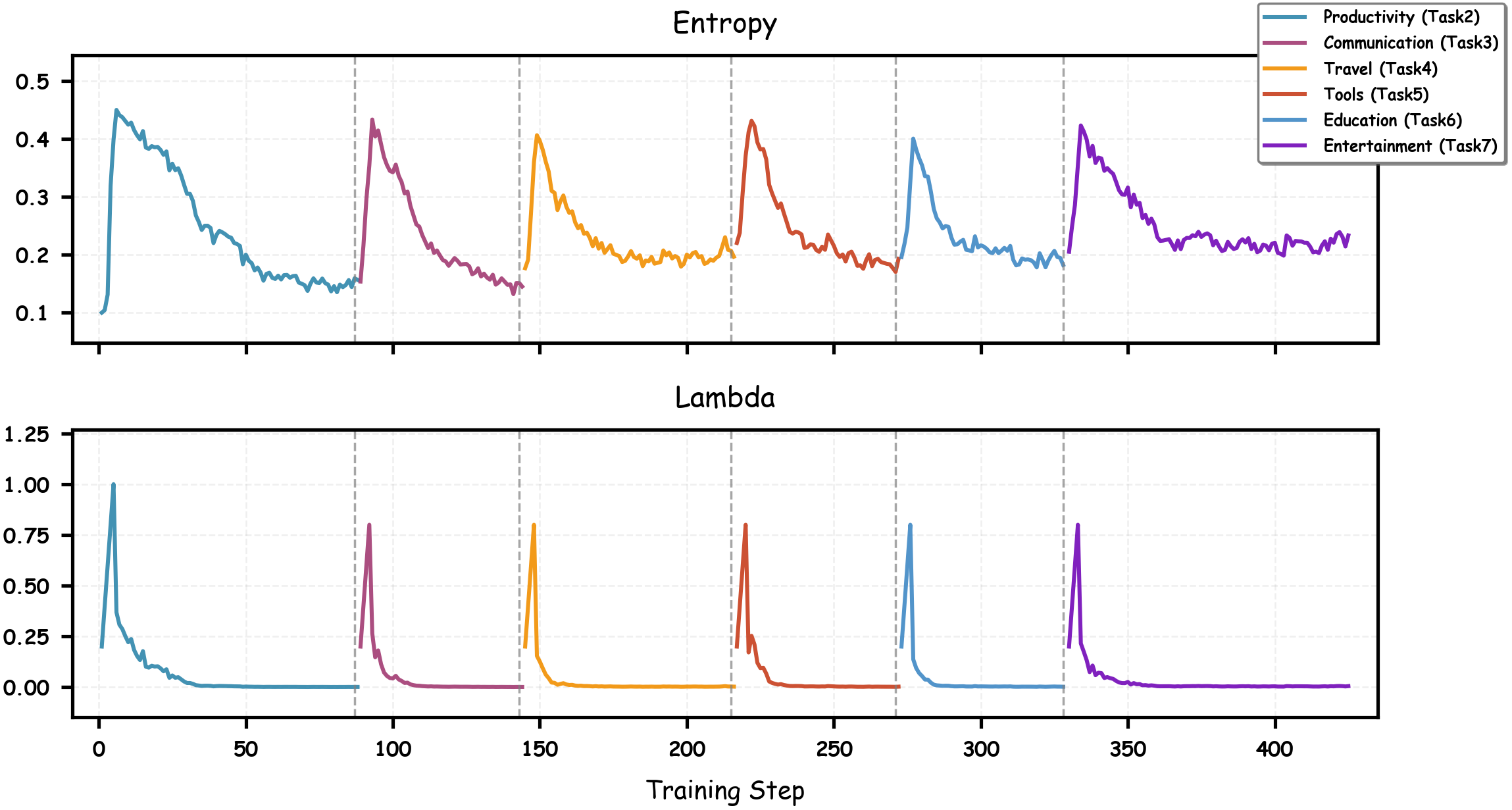}
\vspace{-2mm}
\caption{Trends of policy entropy (top) and SFT loss weight $\lambda$ (bottom) across training steps for Entropy-Regulated Tuning.}

\label{fig:experiment}
\vspace{-10pt}
\end{figure*}

\subsection{Experimental Setup}
\paragraph{Models}
We adopt two MLLMs with different scales as our base models: \textit{QwenVL2.5-3b-Instruct} \cite{bai2025qwen2} and \textit{LLaVA-OneVision-0.5b} \cite{li2024llava}. These models are selected to validate the generality of our method across models with varying capacities.
\vspace{-5mm}
\paragraph{Dataset}
The dataset used in our experiments is the self-proposed \textit{Android-CL} dataset, which is partitioned into 7 app-related tasks. The task categories include Shopping, Communication, Productivity, Travel, System Tools, Education And Science, and Life And Entertainment.
To ensure fair and reliable continual-learning evaluation, we define three training orders representing diverse sequential task permutations:

\begin{enumerate}
    \item Order 1: SP $\rightarrow$ PO $\rightarrow$ CO $\rightarrow$ TT $\rightarrow$ ST $\rightarrow$ ES $\rightarrow$ LE
    \item Order 2: SP $\rightarrow$ LE $\rightarrow$ ES $\rightarrow$ ST $\rightarrow$ TT $\rightarrow$ CO $\rightarrow$ PO
    \item Order 3: ES $\rightarrow$ LE $\rightarrow$ TT $\rightarrow$ SP $\rightarrow$ CO $\rightarrow$ PO $\rightarrow$ ST
\end{enumerate}

\vspace{-5mm}
\paragraph{Implementation Details}
For fair comparison, all methods begin with SFT on the initial task to establish a shared baseline. Subsequent tasks are learned sequentially under their respective continual-learning strategies. Evaluation on the full test sets of all seven tasks assesses both adaptation and forgetting. This design is consistent with existing GUI-agent training pipelines, where many approaches either use SFT alone or apply RL after SFT \cite{luo2025gui, qin2025ui,lu2025ui,hong2024cogagent,bai2024digirl,liu2025infiguiagent}.
All experiments are conducted with PyTorch on 8×Ascend 910B NPUs. For SFT and SFT with replay of histrocal data experiments, we adopt the ms-swift framework \cite{zhao2025swift} with a batch size of 16 and a learning rate of $10^{-5}$. For GRPO \cite{guo2025deepseek}, our proposed CGL method, and RIF-RFT \cite{lai2025reinforcement}, all are implemented based on the verl framework \cite{sheng2025hybridflow} with a batch size of 16, rollout batch size of 512, a learning rate of $10^{-6}$, a KL-divergence coefficient of 0.01, and a sampling group size of 8. For our CGL method, the hyperparameters are set as: $\lambda_{\text{max}} = 1$, $\lambda_{\text{min}} = 0$, $H_{\text{max}} = 0.45$, $\text{step}_{\text{w}} = 5$, $\gamma=20$ and $k = e^{-10}$

\subsection{Main Results}
In the \textbf{Method} section, we propose the \textit{CGL} method, which synergistically integrates SFT and GRPO to balance new-task adaptability and old-task anti-forgetting — a critical challenge for GUI agents that remains underexplored in existing literature. And \textit{RIF-RFT} \cite{lai2025reinforcement} is a recently proposed framework relevant to large-model continual post-training. To comprehensively validate CGL, we conduct experiments using two vision-language models as backbones: lightweight \textit{LLaVA-OneVision-0.5b} \cite{li2024llava} and large-scale \textit{QwenVL2.5-3b-Instruct} \cite{bai2025qwen2}, with \textbf{five baselines} covering diverse technical paradigms:
\begin{itemize}
    \item \textit{SFT}: Standard supervised fine-tuning
    \item \textit{SFT+KL}: SFT augmented with KL-divergence constraints which is same as grpo;
    \item \textit{SFT+Replay}: SFT augmented with 5\% historical data replay for each new task;
    \item \textit{RIF-RFT} \cite{lai2025reinforcement}: A recently proposed data-efficient continual post-training framework;
    \item \textit{GRPO} \cite{guo2025deepseek}: Constrained reinforcement learning for sequential policy updates.
\end{itemize}

\subsubsection{Performance under Task Order 1}
Tab.\ref{tab:method-task-performance-llava-order1}  and Tab.\ref{tab:method-task-performance-qwenvl-order1}  present results under \textbf{Task Order 1}. Consistent conclusions across model scales validate CGL’s superiority in accuracy and anti-forgetting:
\begin{itemize}
    \item \textbf{State-of-the-Art Accuracy}:
    CGL achieves the highest average Step-Acc. and Trajectory-Acc. for both models:
    \begin{itemize}
        \item \textit{QwenVL2.5-3b-Instruct}: 82.33\% Step-Acc. and 38.03\% Trajectory-Acc., outperforming SFT (76.90\%), SFT+KL (80.84\%), SFT+Replay (79.80\%), GRPO (81.53\%), and RIF-RFT (80.44\%);
        \item \textit{LLaVA-OV-One-Vision-0.5b}: 77.84\% Step-Acc. and 24.77\% Trajectory-Acc., surpassing all baselines by 1.27–5.18 percentage points (pp) in Step-Acc.
    \end{itemize}

    \item \textbf{Optimal Anti-Forgetting (Minimal FM)}:
    CGL maintains the smallest FM for both backbones:
    \begin{itemize}
        \item \textit{QwenVL2.5-3b-Instruct}: Near-zero FM (-0.02), reducing forgetting by 0.99 pp (vs. SFT+KL: -1.01) and 0.60 pp (vs. GRPO: -0.62), with negligible old-task knowledge loss;
        \item \textit{LLaVA-OV-One-Vision-0.5b}: FM (-0.52), significantly lower than baselines (SFT: -6.81, GRPO: -0.93, etc.).
    \end{itemize}
\end{itemize}

\subsubsection{Robustness Across Multiple Task Orders}
To evaluate the robustness of our method to task sequence variations, Tab.~\ref{tab:cross-order-performance} presents results across three distinct task orders.
\begin{itemize}
    \item CGL consistently leads in average Step Accuracy (\(\geq 82.33\%\)) and average Trajectory Accuracy (\(\geq 38.03\%\)) across all orders. It outperforms naive SFT (76.87–77.21\%) by 5.40–5.55 percentage points in Step Accuracy and surpasses GRPO (81.44–81.75\%) by 0.80-0.86 percentage points, confirming its strong generalizability to diverse task sequences.
    \item Notably, CGL achieves a \textbf{positive FM (+0.13) in Task Order 2} — a rare outcome in continual learning, as it indicates not only no forgetting of old tasks but also slight performance enhancement. In stark contrast, SFT exhibits severe forgetting with FMs ranging from -5.33 to -6.03 across all orders, while GRPO (FM = -0.33 to -0.62) also suffers from unavoidable performance degradation on previously learned tasks. This highlights CGL’s unique advantage in balancing new-task adaptation and old-task preservation.
\end{itemize}
\begin{figure}
    \centering
    \vspace{-1mm}
    \includegraphics[width=1\linewidth]{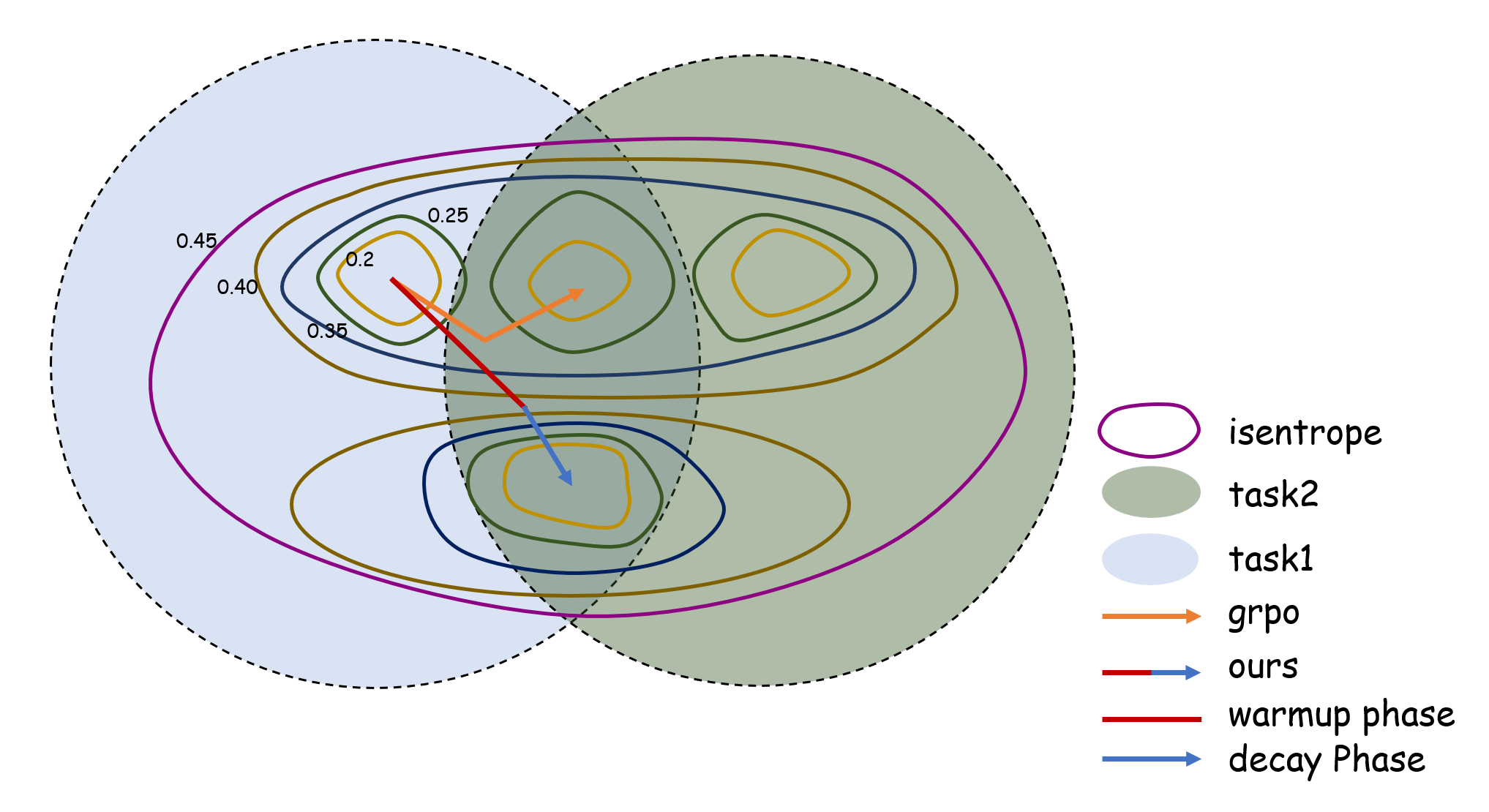}
    \caption{High-accuracy region comparison of GRPO and our CGL in continual GUI learning.}
    \vspace{-5mm}
    \label{fig:CGLgrpo}
\end{figure}

\begin{figure}
    \centering
    \includegraphics[width=1\linewidth]{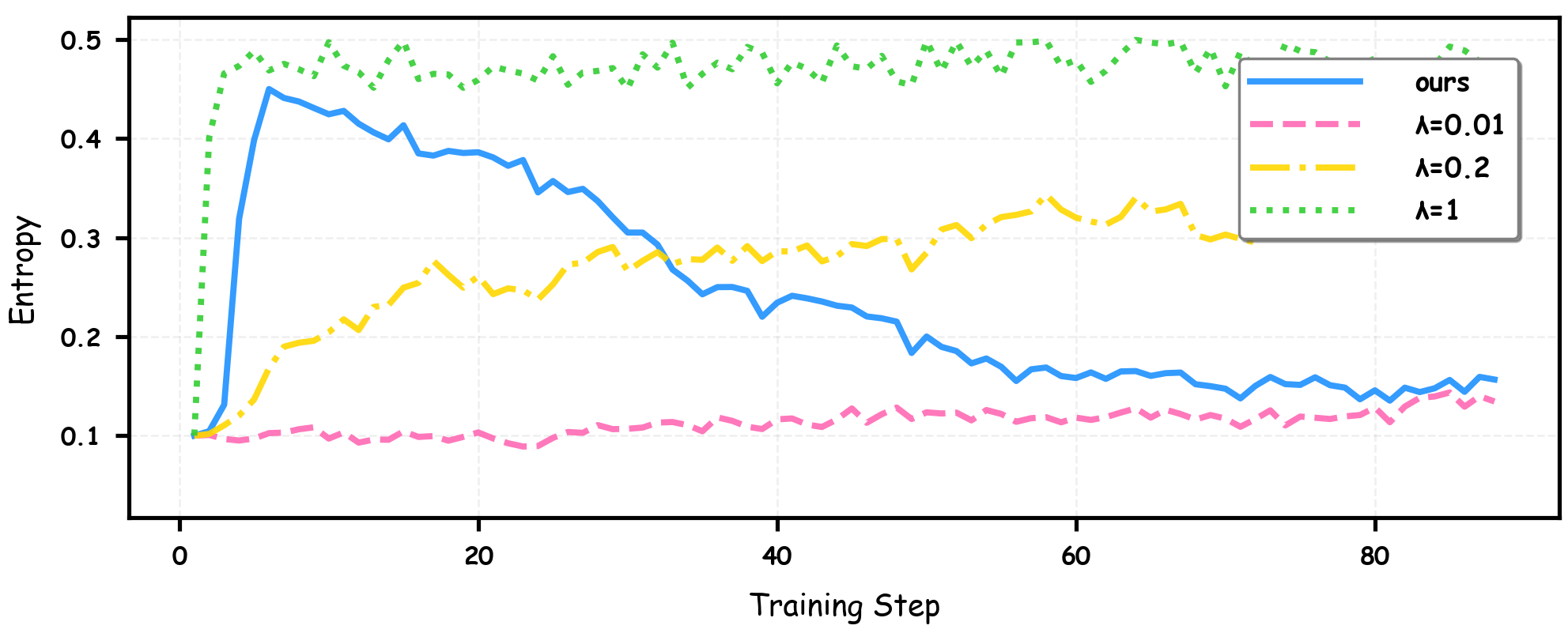}
    \vspace{-6mm}
    \caption{Entropy trends with different hyperparameters}
    \vspace{-3mm}
    \label{fig:entropy_lambda}
\end{figure}

\subsubsection{Ablation Study of Module Contributions}
\paragraph{Analysis of Key Components.}To quantify the individual and synergistic contributions of each component in , we conduct ablation experiments on \textit{QwenVL2.5-3b-Instruct} under Task Order 1, with results presented in Tab.~\ref{tab:ablation-study}. The evaluated modules include SFT, GRPO, KL divergence constraint, dynamic SFT (D-SFT), dynamic SFT weight (D-$\lambda$), and gradient surgery (G-Surg).

We first start with the naive SFT baseline, which achieves 76.90\% Step-Acc, 23.53\% Traj-Acc, and an FM of -5.73. Introducing the KL divergence constraint to SFT improves Step-Acc by 3.94 percentage points (to 80.84\%) and reduces FM to -1.01, verifying the basic anti-forgetting effect of KL constraints. Notably, all three metrics of SFT+KL are still inferior to those of \textit{GRPO}, demonstrating that GRPO’s superior continual learning capability does not solely stem from its KL divergence constraint.

The GRPO module requires the KL divergence constraint to operate stably: without KL, GRPO collapses and performance metrics are unavailable. Combining GRPO with KL elevates Step-Acc to 81.53\% and Traj-Acc to 36.78\%, with FM reduced to -0.62. This demonstrates that constrained reinforcement learning enhances both accuracy and anti-forgetting capability.

We then incrementally integrate additional modules into the SFT+GRPO baseline:
Integrating static full-SFT into GRPO boosts Step-Acc to 81.68\%. Replacing it with D-SFT further improves Step-Acc to 81.90\%, targeting policy weak points to avoid redundant knowledge interference.
Incorporating D-$\lambda$ based on the D-SFT framework further raises Step-Acc to 82.23\% and Traj-Acc to 37.62\%. This module adjusts SFT’s weight via policy entropy and training timesteps, refining the balance between adapting to new tasks and preserving knowledge of old ones.
Introducing Grad-Surg boosts Step-Acc to 82.10\% and reduces FM to -0.05. This module mitigates gradient conflicts between sequential tasks, preventing performance degradation of previously learned knowledge during new task training.
Combining D-$\lambda$ and Grad-Surg elevates Step-Acc to 82.33\% and Traj-Acc to 38.03\%, as the two modules complement each other in weight adjustment and gradient stabilization.
When all modules are integrated (the full CGL framework), we achieve the optimal performance: 82.33\% Step-Acc, 38.03\% Traj-Acc, and an FM of -0.02. This result confirms that each module contributes complementary improvements, and their synergistic integration enables CGL to balance new-task accuracy and old-task anti-forgetting effectively.
\vspace{-6mm}
\paragraph{Analysis of $\lambda$}
As the core of the D-$\lambda$ module, $\lambda$ regulates the injection intensity of SFT error samples, and its dynamic adjustment coupled with policy entropy is key to CGL’s performance.

Trend Correlation (Fig.~\ref{fig:experiment}) depicts the co-variation of policy entropy (top) and $\lambda$ (bottom) across training steps. At task-switching points (vertical dashed lines), $\lambda$ rises linearly in the warmup phase to boost entropy introducing new behaviors via error-sample injection, then decays exponentially in the decay phase to stabilize entropy. This “entropy rise-then-decay” pattern balances new-task adaptability and old-task retention.

Mechanism Advantage: The isentrope diagram (Fig. \ref{fig:CGLgrpo}) illustrates that GRPO is confined to existing isentropic regions, while our entropy-regulated tuning (driven by dynamic $\lambda$) enables crossing isentropes—jumping from low-entropy local optima to more favorable regions, thus avoiding suboptimal policy stagnation.

\begin{table}[tbp]
\vspace{-1mm}
\centering
\caption{Performance comparison under different $\lambda$ settings for QwenVL2.5.}
\vspace{-2mm}
\small
\setlength{\tabcolsep}{5pt} 
\begin{tabular}{lcccccccc}
\toprule
QwenVL2.5 & Avg.Step-Acc. \\
\midrule
$\lambda=1$  & 81.62 \\
$\lambda=0.2$  & 81.58 \\
$\lambda=0.01$ & 81.90 \\
ours & \textbf{82.33} \\
\bottomrule
\end{tabular}
\vspace{-5mm}
\label{tab:lambda-effects}
\end{table}

Performance Validation (Tab. \ref{tab:lambda-effects}) and entropy trends (Fig. \ref{fig:entropy_lambda}): We compare fixed-$\lambda$ settings with our dynamic $\lambda$ (ours):
$\lambda=1$ :excessive error-sample injection leads to overly high entropy.
$\lambda=0.01$ insufficient injection maintains low entropy, trapping the model in local optima.
Our D-$\lambda$ (dynamic $\lambda$) balances injection intensity across tasks, achieving the highest average score (82.33).

\section{Conclusion}
In this paper, we studied the \emph{Continual GUI Learning (CGL)} problem, and proposed a standardized benchmark and a collaborative SFT-RL framework to address it. Specifically, we constructed a CGL benchmark by partitioning the AndroidControl dataset into 7 app-category-specific subsets of comparable scale, filling the gap of unified evaluation in this field. To alleviate the stability–plasticity trade-off between existing methods, we designed two key mechanisms, Entropy-Regulated Tuning and Gradient Surgery for the SFT-RL framework, which balances prior knowledge retention and new skill acquisition. Extensive experiments under three evaluation protocols show that our approach outperforms pure GRPO, pure SFT, and RIF-RFT, demonstrating its effectiveness in the CGL task.
{
    \small
    \bibliographystyle{ieeenat_fullname}
    \bibliography{main}
}

\clearpage
\setcounter{page}{1}
\maketitlesupplementary

\section{Theoretical Analysis of Entropy Dynamics}
\label{sec:theory}

To provide a rigorous theoretical grounding for our Entropy-Regulated Tuning strategy, we analyze the first-order dynamics of policy entropy $\mathcal{H}(\pi_{\theta})$ under the joint influence of GRPO and SFT. Let $z_{s,a}$ denote the logits such that the policy is defined as $\pi_{\theta}(a|s) = \text{softmax}(z_{s,a}) = \frac{\exp(z_{s,a})}{\sum_{a'} \exp(z_{s,a'})}$.

\subsection{Entropy Sensitivity and Separability}

We first establish the relationship between logit updates and entropy variations.

\begin{lemma}[Entropy-Covariance Relationship]
The first-order approximation of the change in policy entropy $\Delta \mathcal{H}$ induced by a logit update vector $\Delta \mathbf{z}_s$ is determined by the negative covariance between the log-probabilities and the update values:
\begin{equation}
\label{eq:entropy_cov}
\begin{aligned}
\Delta \mathcal{H} &= \mathcal{H}(\pi_{\theta}^{k+1}|s) - \mathcal{H}(\pi_{\theta}^{k}|s) \\
&= -\mathrm{Cov}_{a \sim \pi_{\theta}^k} \left( \log \pi_{\theta}^k(a|s), \Delta z_{s,a} \right) + o(\eta)
\end{aligned}
\end{equation}
where $\eta$ is the learning rate.
\end{lemma}

\begin{proof}
The gradient of entropy with respect to logits is $\nabla_{z_{s,a}} \mathcal{H} = -\pi_{\theta}(a|s)(\log \pi_{\theta}(a|s) - \mathcal{H})$. The first-order Taylor expansion gives $\Delta \mathcal{H} \approx \sum_a \nabla_{z_{s,a}} \mathcal{H} \cdot \Delta z_{s,a}$. Substituting the gradient, we get $\Delta \mathcal{H} \approx -\mathbb{E}_{a}[\log \pi_{\theta}(a) \Delta z_{s,a}] + \mathcal{H}\mathbb{E}_{a}[\Delta z_{s,a}]$. This simplifies to the negative covariance definition $-\mathbb{E}[(\log \pi - \mathbb{E}[\log \pi])(\Delta z - \mathbb{E}[\Delta z])]$, assuming centered updates or simplified via the softmax invariance.
\end{proof}

Under our hybrid objective $\mathcal{L} = \mathcal{L}_{GRPO} + \lambda \mathcal{L}_{SFT}$, the total update decomposes linearly: $\Delta z_{total} = \Delta z^{GRPO} + \lambda \Delta z^{SFT}$. Due to the linearity of the covariance operator, the entropy dynamics are separable:
\begin{equation}
\Delta \mathcal{H}_{total} \approx \Delta \mathcal{H}_{GRPO} + \lambda \Delta \mathcal{H}_{SFT}
\end{equation}
This property allows us to analyze the impact of GRPO and SFT independently.

\subsection{Entropy Decay in GRPO: The Matthew Effect}
\label{Matthew}
The GRPO algorithm optimizes the expected relative advantage. We show that this naturally leads to entropy reduction.

\begin{lemma}[GRPO Update Dynamics]
The logit update under GRPO, denoted as $\Delta z_{s,a}^{GRPO}$, is proportional to the probability-weighted advantage:
\begin{equation}
\label{eq:grpo_update}
\Delta z_{s,a}^{GRPO} = \eta \cdot \pi_{\theta}(a|s) \cdot A(s,a)
\end{equation}
assuming the mean advantage baseline is handled implicitly.
\end{lemma}

\begin{proof}
The gradient of the GRPO objective $J = \mathbb{E}_{a \sim \pi}[A(s,a)]$ with respect to logit $z_{s,a}$ is:
\begin{equation}
\begin{aligned}
\frac{\partial J}{\partial z_{s,a}} &= \mathbb{E}_{a' \sim \pi} \left[ \frac{\partial \log \pi(a'|s)}{\partial z_{s,a}} A(s,a') \right] \\
&= \sum_{a'} \pi(a'|s) (\mathbb{I}[a=a'] - \pi(a|s)) A(s,a') \\
&= \pi(a|s) \left( A(s,a) - \sum_{a'} \pi(a'|s) A(s,a') \right)
\end{aligned}
\end{equation}
In GRPO, advantages are normalized group-wise such that $\sum \pi A \approx 0$. Thus, the update simplifies to $\Delta z_{s,a} = \eta \frac{\partial J}{\partial z_{s,a}} \approx \eta \pi(a|s) A(s,a)$.
\end{proof}

\textbf{Analysis (The Matthew Effect):} 
Eq.~\eqref{eq:grpo_update} reveals a self-reinforcing mechanism. Actions with initially high probability $\pi(a|s)$ and positive advantage $A(s,a)>0$ receive the largest positive updates.
\begin{itemize}
    \item High $\pi(a|s)$ implies high $\log \pi(a|s)$.
    \item High $\pi(a|s)$ leads to large positive $\Delta z_{s,a}$.
\end{itemize}
This creates a \textbf{positive correlation} between the distribution and the update: $\mathrm{Cov}(\log \pi, \Delta z^{GRPO}) > 0$. According to Lemma 1, this results in $\Delta \mathcal{H}_{GRPO} < 0$. This mathematically explains the "Stage 3" behavior where the model converges sharply to a low-entropy state.

\subsection{Entropy Injection via Error-Driven SFT}

When the model exhibits pathological bias (i.e., convergence to a wrong action), we activate SFT on the ground truth $a^*$.

\subsection{Entropy Injection via Error-Driven SFT}

When the model suffers from pathological bias (Stage 1), we activate SFT on the ground-truth action $a^*$.

\begin{lemma}[SFT Entropy Injection]
Let the SFT update follow the gradient of the target probability likelihood. The logit update $\Delta z^{SFT}$ satisfies the zero-sum property and induces entropy injection by creating a negative covariance with the current distribution.
\end{lemma}

\begin{proof}
Consider the optimization of the probability mass $\pi_{\theta}(a^*|s)$ directly. Utilizing the derivative property of the Softmax function, $\frac{\partial \pi_i}{\partial z_j} = \pi_i (\delta_{ij} - \pi_j)$, the update for logit $z_{s,a}$ is derived as:
\begin{equation}
\label{eq:sft_update}
\begin{aligned}
\Delta z_{s,a}^{SFT} &= \eta \cdot \frac{\partial \pi_{\theta}(a^*|s)}{\partial z_{s,a}} \\
&= \eta \cdot \pi_{\theta}(a^*|s) \cdot \left( \mathbb{I}[a=a^*] - \pi_{\theta}(a|s) \right)
\end{aligned}
\end{equation}
We verify the zero-sum property of this update:
\begin{equation}
\sum_{a} \Delta z_{s,a}^{SFT} = \eta \pi_{\theta}(a^*|s) \left( \sum_a \mathbb{I}[a=a^*] - \sum_a \pi_{\theta}(a|s) \right) = 0
\end{equation}
This confirms the update purely redistributes probability mass. We then analyze the correlation structure in the error state ($\pi_{\theta}(a^*|s) \to 0$):
\begin{enumerate}
    \item \textbf{For the target $a=a^*$:} The current log-probability $\log \pi_{\theta}$ is minimal (negative), while the update component $(\mathbb{I} - \pi_{\theta}) > 0$ is positive.
    \item \textbf{For erroneous actions $a \neq a^*$:} The current log-probability $\log \pi_{\theta}$ is high (confident error), while the update component $(0 - \pi_{\theta}) < 0$ is negative.
\end{enumerate}
Despite the scaling factor $\pi_{\theta}(a^*|s)$, the direction of $\Delta z$ is strictly opposite to the bias of $\pi_{\theta}$. Thus, $\mathrm{Cov}(\log \pi_{\theta}, \Delta z^{SFT}) < 0$.
\end{proof}

\textbf{Analysis:}
According to Lemma 1, the negative covariance implies $\Delta \mathcal{H}_{SFT} > 0$. By increasing the weight $\lambda$ in \textbf{Stage 2}, the framework amplifies this corrective update. Even if the magnitude is initially dampened by the small $\pi_{\theta}(a^*|s)$, the directional alignment forces the "heating up" of the distribution, breaking the local minima of the initial error.

\section{Reward Function For Gui Agent}
\label{reward_ui}
To guide effective policy learning, we categorize UI actions into three classes with tailored reward rules.

\begin{itemize}
    \item \textbf{Navigation and State Actions}: Include \texttt{Home()}, \texttt{BACK()}, \texttt{WAIT()}, \texttt{Finish()}.Reward: +1 for exact action type match with ground truth (GT); 0 otherwise.
    
    \item \textbf{Parameterized Actions}: Include \texttt{INPUT\_TEXT(\textit{str})}, \texttt{SCROLL(\textit{direction})},\texttt{OPEN\_APP(\textit{app\_name})}.Reward: +1 for type match (Stage 1),+ 1 for exact parameter match (Stage 2), total maximum +2.
    
    \item \textbf{Spatial Interaction Actions}: Include \texttt{CLICK([x,y])}, \texttt{LONG\_PRESS([x,y])}.Reward: +1 for type match (Stage 1) + 1 if predicted coordinates fall within GT element’s bounding box (Stage 2), total maximum +2.
\end{itemize}

\begin{table*}[!h]
  \centering
  \caption{App Category Statistics for Training Dataset}
  \label{tab:app_category_stats}
  \begin{tabular}{c c c l}
    \toprule
    Category         & Number of Apps  & Number of Episodes & App Examples \\
    \midrule
    Shopping         & 15                          & 1140                               & amazon, ebay, alibaba.com \\
    Productivity     & 15                          & 1240                               & Drive, WPS Office, Word \\
    Communication    & 8                           & 617                                & Gmail, Outlook, Skype \\
    Travel           & 10                          & 686                                & Maps, booking.com, Expedia \\
    Tools            & 13                          & 771                                & Clock, Recorder, Settings \\
    Education        & 11                          & 559                                & Leafsnap, Arts \& Culture, Duolingo \\
    Entertainment    & 17                          & 1030                               & YouTube, Vimeo, Pinterest \\
    \bottomrule
  \end{tabular}
\end{table*}

\begin{table*}[!htbp]
  \centering
  \caption{\textbf{Evolution of Task-Specific Performance (LLaVA-OneVision-0.5b, Task Order 1).} Columns represent the sequence of training tasks (left to right); rows represent the evaluation performance on specific test sets after each training stage.}
  \label{tab:continual_learning_results_llava}
  \setlength{\tabcolsep}{4pt}
  \begin{tabular}{l c c c c c c c}
    \toprule
    \textbf{Training Stage} $\rightarrow$ & Shopping & Productivity & Communication & Travel & Tools & Education & Entertainment \\
    \midrule
    \textit{Test on:} & & & & & & & \\
    Shopping & 77.63  & 75.63       & 75.58       & 76.58       & 76.12       & 76.49       & 75.55      \\
    Productivity & -        & 80.80       & 80.80       & 78.39       & 80.97       & 81.30       & 81.32      \\
    Communication &-          &-            & 80.88       & 77.90       & 80.77       & 80.33       & 80.73      \\
    Travel & -        &    -         & -            & 77.78       & 77.09       & 76.27       & 76.88      \\
    Tools &    -      &    -         & -            &    -        & 84.14       & 84.09       & 83.41      \\
    Education & -        &    -         &        -      &    -        &    -        & 69.11       & 68.90      \\
    Entertainment &       -    & -            &    -        &    -        &  -           &      -        & 78.13      \\
    \bottomrule
  \end{tabular}
\end{table*}

\begin{table*}[!htbp]
  \centering
  \caption{\textbf{Evolution of Task-Specific Performance (CGL on Qwen2.5-VL-3b, Task Order 1).} CGL demonstrates stability and positive transfer (e.g., Tools accuracy improves over time).}
  \label{tab:continual_learning_results_CGL}
  \setlength{\tabcolsep}{4pt}
  \begin{tabular}{l c c c c c c c}
    \toprule
    \textbf{Training Stage} $\rightarrow$ & Shopping & Productivity & Communication & Travel & Tools & Education & Entertainment \\
    \midrule
    \textit{Test on:} & & & & & & & \\
    Shopping & 79.77  & 80.67       & 81.33       & 81.77       & 80.53       & 81.72       & 80.32      \\
    Productivity &  -        & 85.41       & 85.18       & 84.13       & 85.46       & 85.29       & 85.46      \\
    Communication &    -      &  -           & 84.23       & 83.98       & 82.98       & 83.20       & 81.99      \\
    Travel &       -    &    -          &       -       & 83.06       & 83.13       & 82.51       & 82.03      \\
    Tools &  -        &    -          &    -          &      -        & 86.42       & 87.88       & 88.17      \\
    Education &       -    &    -         & -            &    -          &  -           & 75.05       & 75.81      \\
    Entertainment &    -      &    -         &        -      &    -        &  -           &        -      & 82.53      \\
    \bottomrule
  \end{tabular}
\end{table*}
\section{Details of Benchmark}

To facilitate a deeper understanding of the \textbf{AndroidControl-CL} benchmark, we provide a comprehensive gallery of episode trajectories in Figs. \ref{fig:task1_ex1} through \ref{fig:task7_ex3}. Unlike previous datasets that often rely on homogenous patterns, our benchmark is designed to capture the heterogeneity of real-world mobile interactions.

Specifically, for each of the 7 functional task categories (Shopping, Productivity, Communication, Travel, Tools, Education, and Entertainment), we showcase \textbf{three distinct episodes sourced from different applications}. These visualizations highlight several key characteristics of the benchmark:
\begin{itemize}
    \item \textbf{Cross-App Diversity:} Even within the same category (\eg Shopping), different apps (Amazon, eBay, Flipkart) exhibit unique UI layouts, search logic, and interaction flows (see Figs. \ref{fig:task1_ex1}--\ref{fig:task1_ex3}).
    \item \textbf{Task Complexity:} The examples range from simple information retrieval (\eg checking weather in Fig. \ref{fig:task5_ex2}) to complex, multi-step workflows (\eg booking a flight with specific constraints in Fig. \ref{fig:task4_ex2}).
    \item \textbf{Action Granularity:} The visualizations demonstrate the precise coordinate-based actions and text inputs required, validating the necessity for fine-grained control in GUI agents.
\end{itemize}

For quantitative details regarding the dataset composition, please refer to Tab~\ref{tab:app_category_stats}. The balanced distribution—ranging from 559 episodes in Education to 1240 in Productivity—ensures that the evaluation of continual learning performance is not biased toward dominant application domains.

\section{Additional Experimental Results}
This section provides further empirical evidence supporting the effectiveness of the CGL framework compared to baselines (GRPO, SFT). We analyze performance across models of varying parameter scales and contrast sequential learning with multi-task joint training. Unless otherwise stated, Task Order 1 is adopted as the evaluation protocol.

\subsection{CGL Performance Across Models of Varying Scales}
To validate the scalability and generality of the CGL framework, we conducted experiments on two multimodal models with distinct parameter sizes: the lightweight \textbf{LLaVA-OneVision-0.5b} and the large-scale \textbf{Qwen2.5-VL-3b}.

\noindent\textbf{Performance on Lightweight Model (0.5B):} 
Tab. \ref{tab:continual_learning_results_llava} presents the step-wise accuracy evolution for LLaVA-OneVision-0.5b. Despite the limited model capacity, CGL maintains remarkable stability. For instance, after the initial training on the \textit{Shopping} task, the model preserves high performance on subsequent evaluations even as it adapts to new domains (\eg retaining $\sim$76\% on Shopping after learning Education), demonstrating effective mitigation of forgetting in resource-constrained settings.

\noindent\textbf{Performance on Large-scale Model (3B):} 
As shown in Tab. \ref{tab:continual_learning_results_CGL}, CGL demonstrates superior capability on Qwen2.5-VL-3b, achieving not only retention but also \textbf{positive backward transfer}. Key tasks such as \textit{Shopping}, \textit{Tools}, and \textit{Productivity} exhibit performance maintenance or improvement as new tasks are introduced. Specifically, the accuracy on \textit{Tools} increases from 86.42\% to 88.17\%, and \textit{Productivity} improves marginally from 85.29\% to 85.46\%. This suggests that the synergy between SFT and GRPO in CGL enables the consolidation of generalizable GUI interaction logic.

Quantitatively, CGL achieves a near-zero Forgetting Measure (FM = -0.02) under Task Order 1. Notably, under Task Order 2 (Tab. \ref{tab:continual_learning_results_CGL_order2}), CGL achieves a positive FM (+0.13). This indicates that the learning of subsequent tasks actively reinforces the representations of previously learned tasks, a phenomenon rarely observed in standard continual learning baselines.

\subsection{Baseline Analysis: GRPO and SFT}
To contextualize the performance of CGL, we analyze the behavior of baseline methods on Qwen2.5-VL-3b:

\begin{itemize}
    \item \textbf{GRPO (Tab.~\ref{tab:continual_learning_results_GRPO}):} While GRPO demonstrates strong stability and resistance to catastrophic forgetting, it lacks the rapid adaptation capabilities observed in CGL, serving primarily as a stability anchor.
    \item \textbf{SFT (Tab.~\ref{tab:continual_learning_results_SFT}):} Standard Supervised Fine-Tuning exhibits severe plasticity-stability trade-offs. For instance, after adapting to the \textit{Entertainment} task, performance on early tasks such as \textit{Shopping} drops significantly (79.77\% $\rightarrow$ 70.54\%), confirming that SFT prioritizes new task acquisition at the expense of catastrophic forgetting.
\end{itemize}

\subsection{Comparison with Multi-Task Joint Training}
We further compare our sequential learning approach against the theoretical upper bound of Multi-Task Joint Training (Tab.~\ref{tab:training_perf_transposed}). Joint training assumes simultaneous access to all datasets, which is typically impractical for dynamic, real-world GUI agents.

\textbf{Remarkably, CGL significantly narrows the gap between sequential learning and joint training.} As illustrated in Tab. \ref{tab:training_perf_transposed}:
\begin{itemize}
    \item The average trajectory accuracy of CGL across three different task orders is \textbf{82.33\%} (Order 1), \textbf{82.61\%} (Order 2), and \textbf{82.40\%} (Order 3).
    \item These results are highly competitive with the \textbf{GRPO-Joint-Training} baseline, which achieves \textbf{82.41\%}. Notably, our method under Task Order 2 even slightly outperforms the GRPO joint training baseline.
    \item While \textbf{SFT-Joint-Training} reaches the highest absolute accuracy (83.48\%), it requires full data accessibility. CGL achieves comparable performance to the RL-based joint training upper bound while strictly adhering to the sequential learning constraint, validating the efficacy of our Entropy-Regulated Tuning and Gradient Surgery mechanisms in maintaining policy optimality over time.
\end{itemize}


\begin{table*}[htbp]
  \centering
  \caption{\textbf{Evolution of Task-Specific Performance (CGL on Qwen2.5-VL-3b, Task Order 2).} The distinct task order further validates the robustness of CGL.}
  \label{tab:continual_learning_results_CGL_order2}
  \setlength{\tabcolsep}{4pt}
  \begin{tabular}{l c c c c c c c}
    \toprule
    \textbf{Training Stage} $\rightarrow$ & Shopping & Entertainment  & Education & Tools &  Travel&Communication  & Productivity  \\
    \midrule
    \textit{Test on:} & & & & & & & \\
    Shopping & 79.77 & 79.96  & 80.53    & 80.72       & 80.58       & 80.25       & 79.63            \\
    Entertainment&    -       & 81.55       & 81.45       & 81.05       & 80.71       & 81.45       & 82.08      \\
    Education  &       -    &      -        & 76.67       & 75.81       & 76.03       & 75.27       & 75.81      \\
    Tools & -        &       -       &         -    & 86.52       & 86.23       & 86.61       & 87.29      \\
    Travel &  -        &    -         &    -        &       -       & 81.96       & 82.10       & 82.85      \\
    Communication &    -      &    -          &       -       &  -           &    -        & 84.75       & 84.42      \\
    Productivity & -        &    -          &        -      &    -        &    -        &        -      & 86.18      \\
    \bottomrule
  \end{tabular}
\end{table*}

\begin{table*}[htbp]
  \centering
  \caption{\textbf{Baseline Performance: GRPO (Qwen2.5-VL-3b, Task Order 1).} GRPO shows high stability but lower plasticity compared to CGL.}
  \label{tab:continual_learning_results_GRPO}
  \setlength{\tabcolsep}{4pt}
  \begin{tabular}{l c c c c c c c}
    \toprule
    \textbf{Training Stage} $\rightarrow$ & Shopping & Productivity & Communication & Travel & Tools & Education & Entertainment \\
    \midrule
    \textit{Test on:} & & & & & & & \\
    Shopping & 79.77  & 81.39       & 80.30       & 80.34       & 79.72       & 80.34       & 79.63      \\
    Productivity &  -        & 85.22       & 84.18       & 83.63       & 84.13       & 84.46       & 84.74      \\
    Communication &  -        &  -           & 84.20       & 82.65       & 83.20       & 82.87       & 83.43      \\
    Travel &    -      &      -        &    -          & 81.98       & 80.59       & 80.79       & 81.34      \\
    Tools &    -      &        -      &      -        &       -       & 86.71       & 85.35       & 85.74      \\
    Education &       -    &        -      &         -    &       -       &       -       & 75.05       & 74.08      \\
    Entertainment &    -      &         -    &         -    &    -        &        -      &          -    & 81.73      \\
    \bottomrule
  \end{tabular}
\end{table*}

\begin{table*}[htbp]
  \centering
  \caption{\textbf{Baseline Performance: SFT (Qwen2.5-VL-3b, Task Order 1).} SFT suffers from significant forgetting (e.g., Shopping drops from 79.77\% to 70.54\%).}
  \label{tab:continual_learning_results_SFT}
  \setlength{\tabcolsep}{4pt}
  \begin{tabular}{l c c c c c c c}
    \toprule
    \textbf{Training Stage} $\rightarrow$ & Shopping & Productivity & Communication & Travel & Tools & Education & Entertainment \\
    \midrule
    \textit{Test on:} & & & & & & & \\
    Shopping & 79.77  & 75.62       & 73.59       & 73.34       & 71.57       & 70.89       & 70.54      \\
    Productivity &       -    & 86.49       & 82.97       & 81.92       & 80.89       & 80.12       & 79.52      \\
    Communication &    -      &    -          & 85.30       & 81.80       & 80.23       & 79.67       & 78.90      \\
    Travel &       -    &       -       &    -          & 82.06       & 79.78       & 78.45       & 76.41      \\
    Tools &    -      &         -    &    -        &  -           & 87.81       & 85.01       & 83.24      \\
    Education &    -      &      -        &         -    &        -      &      -        & 74.89       & 69.24      \\
    Entertainment &    -      &       -       &         -    &       -       &      -        &         -    & 80.42      \\
    \bottomrule
  \end{tabular}
\end{table*}

\begin{table*}[htbp]
  \centering
  \caption{\textbf{Comparison with Multi-Task Joint Training (Qwen2.5-VL-3b).} Joint training represents the theoretical upper bound where all data is available simultaneously. CGL (Ours) achieves performance comparable to GRPO-Joint-Training across all task orders.}
  \resizebox{\linewidth}{!}{
  \begin{tabular}{lcccccccc} 
    \toprule
    Training Method        & Shopping & Productivity & Communication & Travel & Tools & Education & Entertainment & Avg. Trajectory Acc. \\ 
    \midrule
    Ours-Order1                  & 80.32     & 85.46        & 81.99          & 82.03  & \textbf{88.17}  &  \underline{75.81}   &  \underline{82.53} & 82.33\\
    Ours-Order2                  & 79.63     & \underline{86.18}        & 84.42 
            & \textbf{82.85}   & 87.29  &  \underline{75.81}   & 82.08 & \underline{82.61} \\
    Ours-Order3                 & 80.39  & 86.17  &  \underline{84.91} & 81.81          &  87.87  &  74.16  & 81.50 & 82.40 \\
    SFT-Joint-Training     & \textbf{81.73}    & \textbf{86.74}        & \textbf{85.75}          & 82.24  & \textbf{88.17} & \textbf{77.00}      & \textbf{82.76}   & \textbf{83.48}      \\
    GRPO-Joint-Training  & \underline{80.58}    & 85.97        & 84.53          & \underline{82.65}  & 87.49 & 74.73      & 80.94  & 82.41        \\
    \bottomrule
  \end{tabular}}
  \label{tab:training_perf_transposed}
\end{table*}


\begin{figure*}[htbp]
  \centering
  \includegraphics[height=5cm, keepaspectratio]{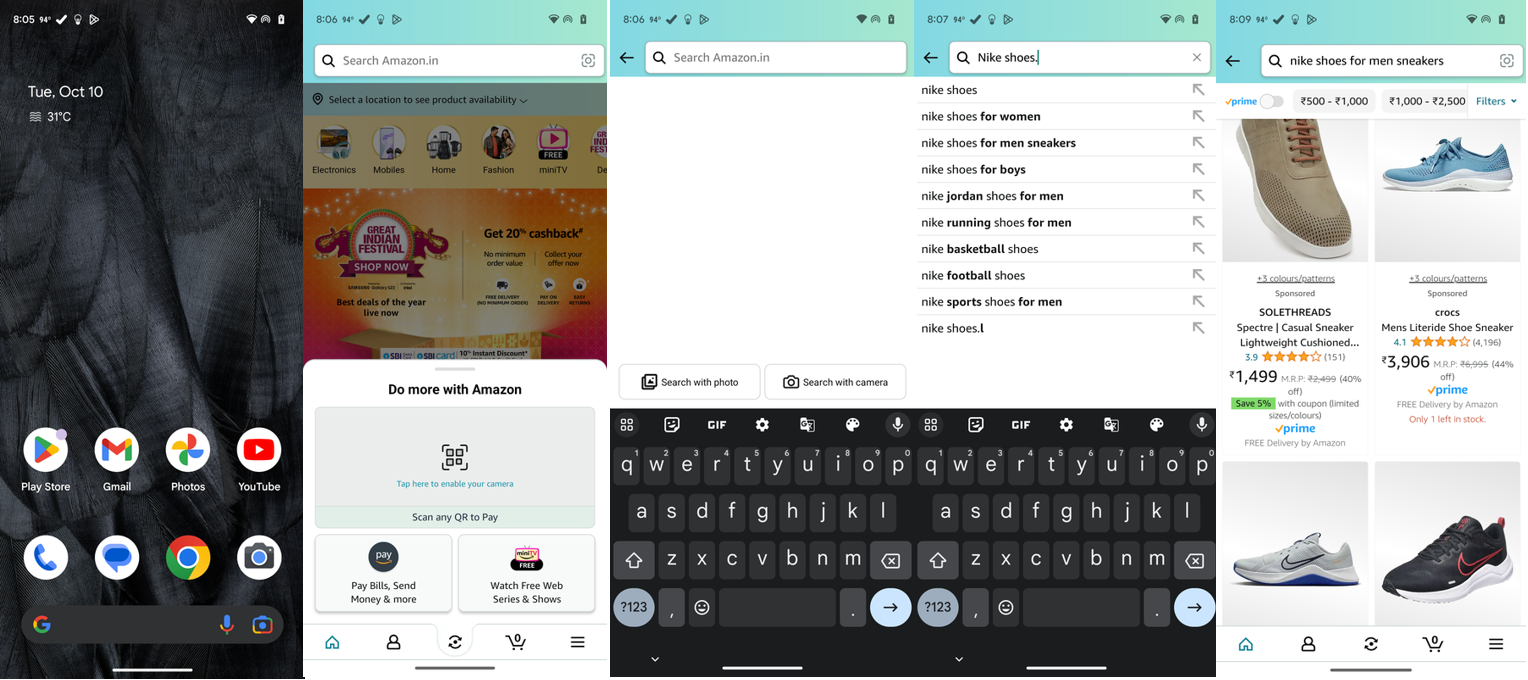} 
  \caption{\textbf{Trajectory Visualization: Search for New Nike Shoes on Amazon.} 
  The agent executes the task through the following action sequence: 
  (1) Launch the application via \texttt{Open(Amazon)}; 
  (2) Focus on the search field via \texttt{Click(0.507, 0.089)}; 
  (3) Input the product query via \texttt{Input\_text(``Nike shoes'')}; 
  (4) Select the target suggestion via \texttt{Click(0.444, 0.210)}; 
  (5) Browse the product list via \texttt{Scroll(down)}; 
  (6) Terminate the episode via \texttt{Finish()}. 
  Coordinates indicate the relative $(x, y)$ position on the screen interface.}
  \label{fig:task1_ex1}
\end{figure*}

\begin{figure*}[htbp]
  \centering
  \includegraphics[height=5cm, keepaspectratio]{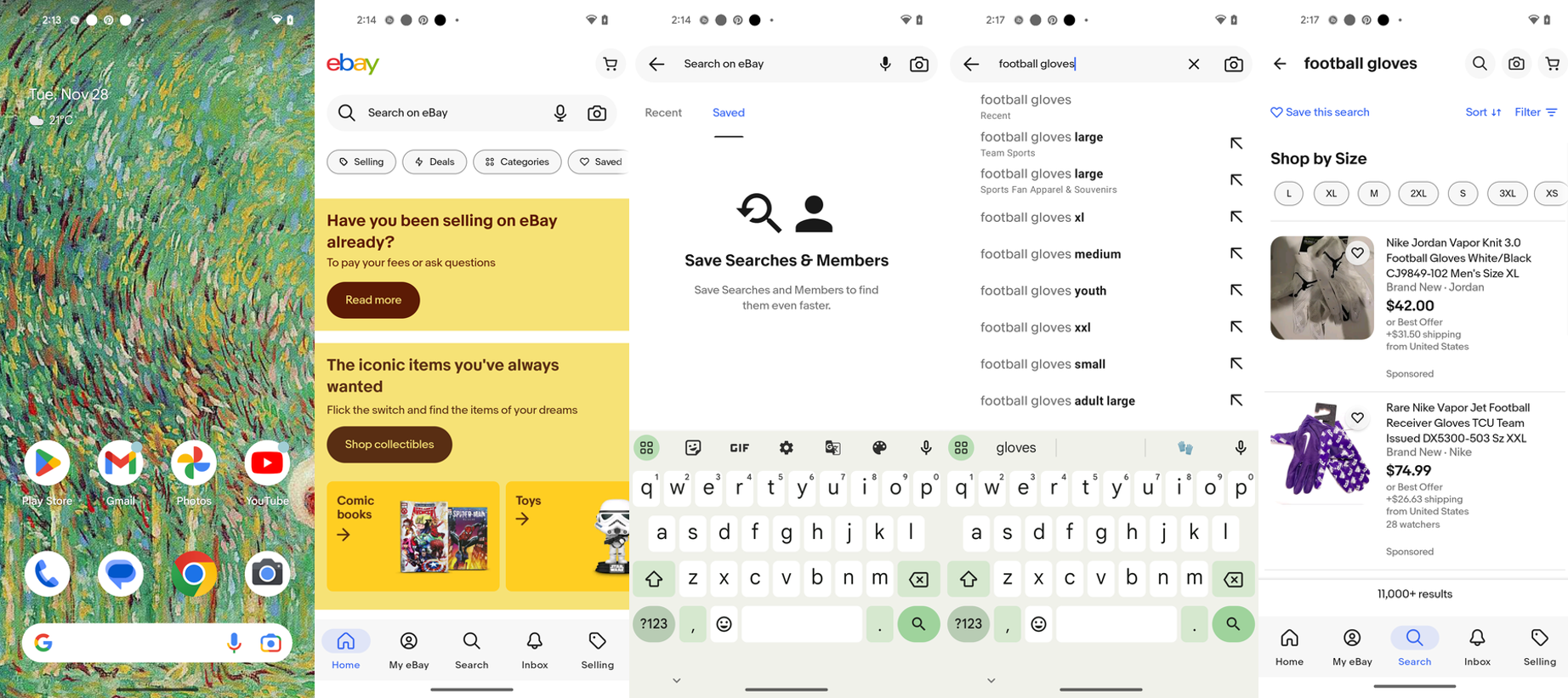} 
  \caption{\textbf{Trajectory Visualization: Product Search on eBay.} 
  The agent locates specific sporting goods via the following sequence: 
  (1) Launch the application via \texttt{Open(eBay)}; 
  (2) Activate the search interface via \texttt{Click(0.382, 0.162)}; 
  (3) Input the product category via \texttt{Input\_text(``football gloves'')}; 
  (4) Select the top suggestion via \texttt{Click(0.490, 0.152)}; 
  (5) Terminate the task via \texttt{Finish()}.}
  \label{fig:task1_ex2}
\end{figure*}

\begin{figure*}[htbp]
  \centering
  \includegraphics[height=5cm, keepaspectratio]{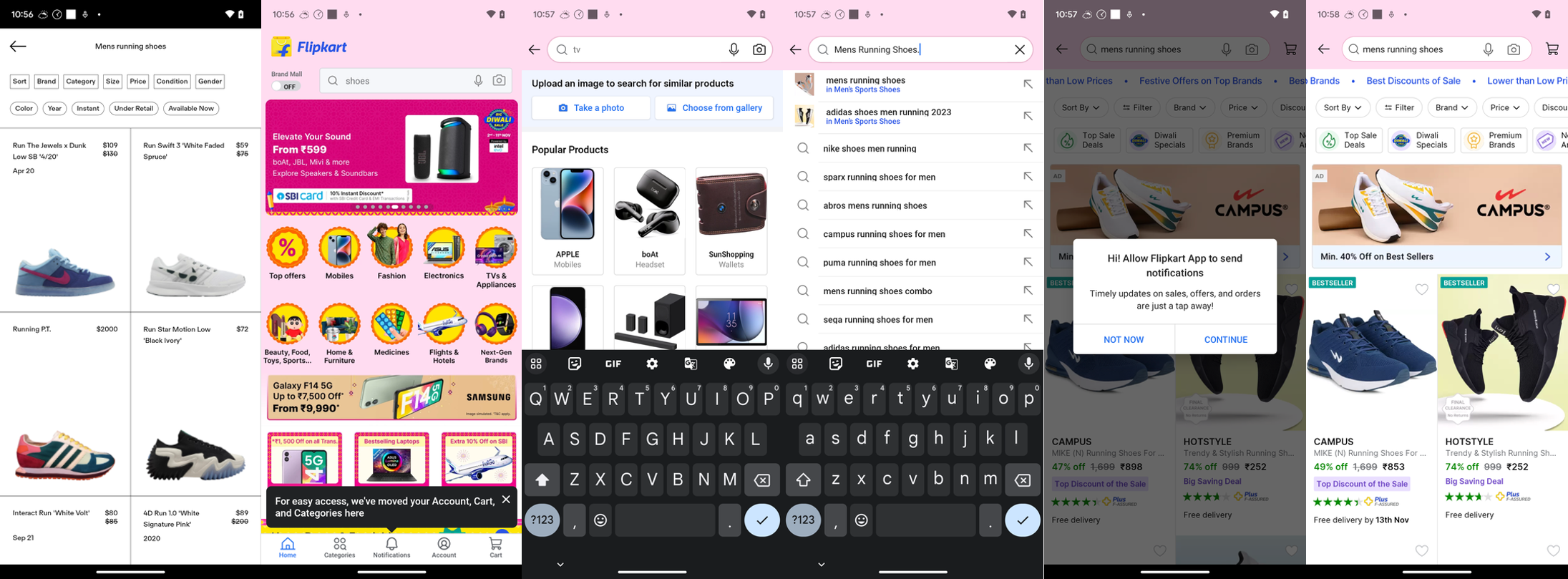} 
  \caption{\textbf{Trajectory Visualization: Search for Mens Running Shoes on Flipkart.} 
  The agent successfully navigates the interface through the following action sequence: 
  (1) Launch the application via \texttt{Open(Flipkart)}; 
  (2) Activate the search field via \texttt{Click(0.562, 0.139)}; 
  (3) Enter the search query via \texttt{Input\_text(``Mens Running Shoes'')}; 
  (4) Select the target product from the results via \texttt{Click(0.318, 0.154)}; 
  (5) Complete the task via \texttt{Finish()}. 
  Coordinates $(x, y)$ represent the normalized relative position on the screen.}
  \label{fig:task1_ex3}
\end{figure*}

\begin{figure*}[htbp]
  \centering
  \includegraphics[height=5cm, keepaspectratio]{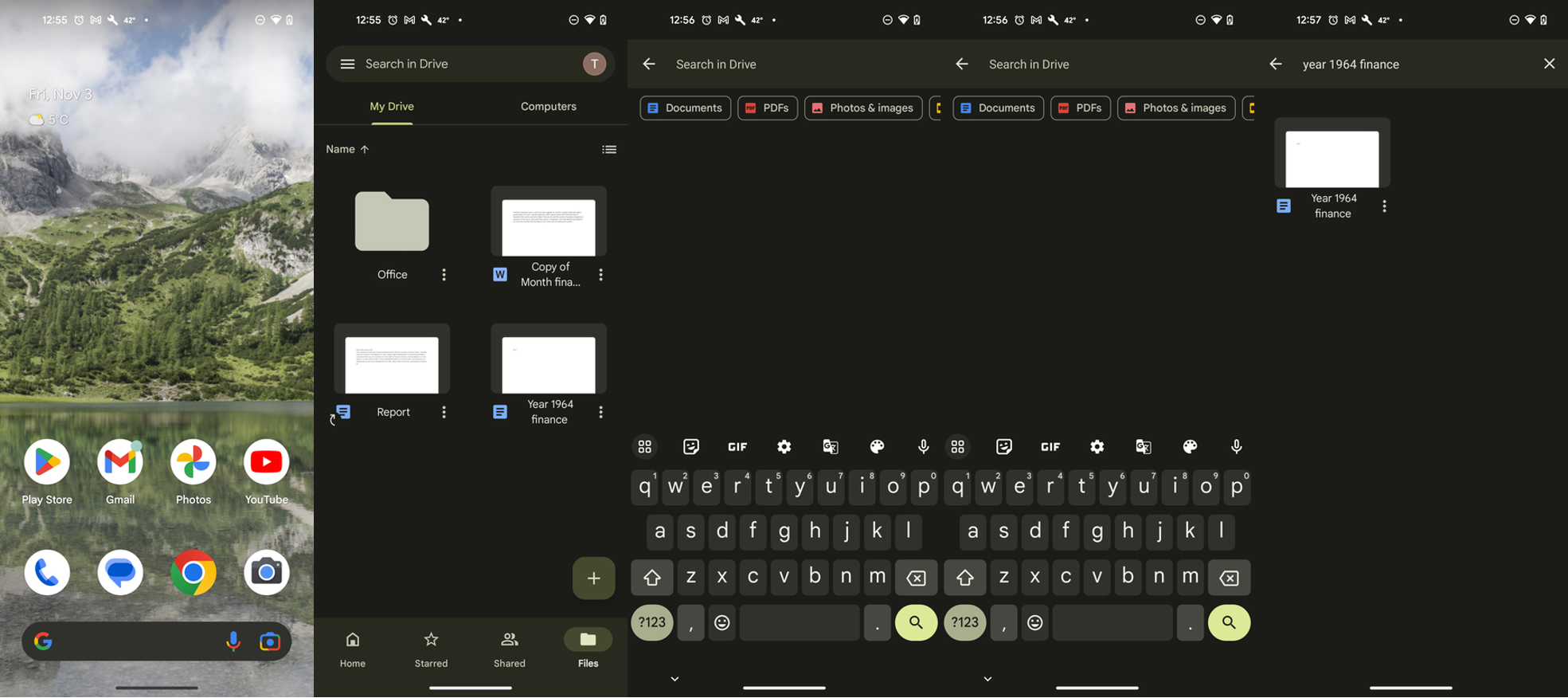} 
  \caption{\textbf{Trajectory Visualization: File Retrieval on Google Drive.} 
  The agent retrieves a specific historical document: 
  (1) Launch the application via \texttt{Open(Drive)}; 
  (2) Focus on the search bar via \texttt{Click(0.500, 0.091)}; 
  (3) Enter the file query via \texttt{Input\_text(``Year 1964 finance'')}; 
  (4) Access the target file from results via \texttt{Click(0.495, 0.159)}; 
  (5) Terminate the task via \texttt{Finish()}.}
  \label{fig:task2_ex1}
\end{figure*}

\begin{figure*}[htbp]
  \centering
  \includegraphics[height=5cm, keepaspectratio]{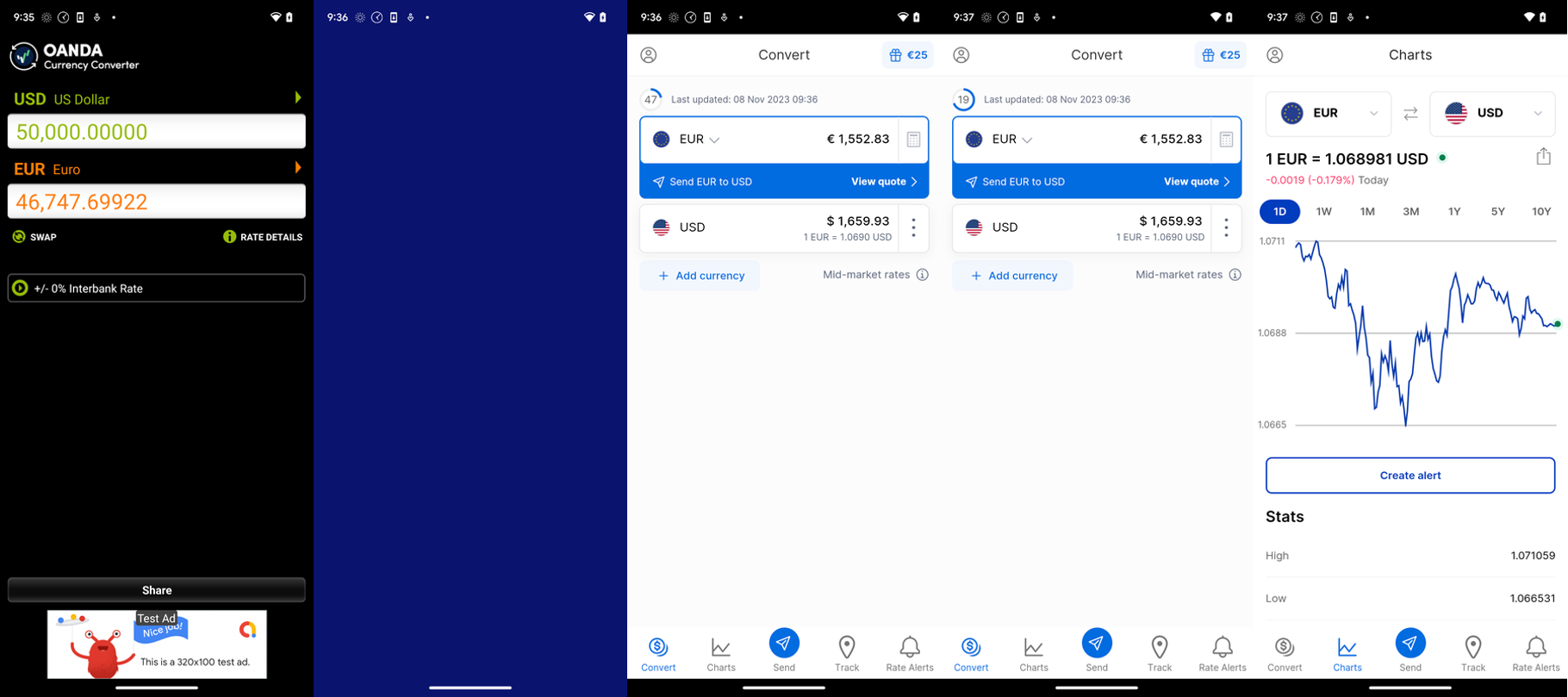} 
  \caption{\textbf{Trajectory Visualization: Currency Exchange Tracking on Xe.} 
  The agent navigates to the financial charting interface: 
  (1) Launch the application via \texttt{Open(Xe)}; 
  (2) Synchronize with the interface via \texttt{Wait()}; 
  (3) Execute a secondary \texttt{Wait()} for data loading; 
  (4) Navigate to the charts section via \texttt{Click(0.300, 0.937)}; 
  (5) Terminate the task via \texttt{Finish()}.}
  \label{fig:task2_ex2}
\end{figure*}

\begin{figure*}[htbp]
  \centering
  \includegraphics[height=5cm, keepaspectratio]{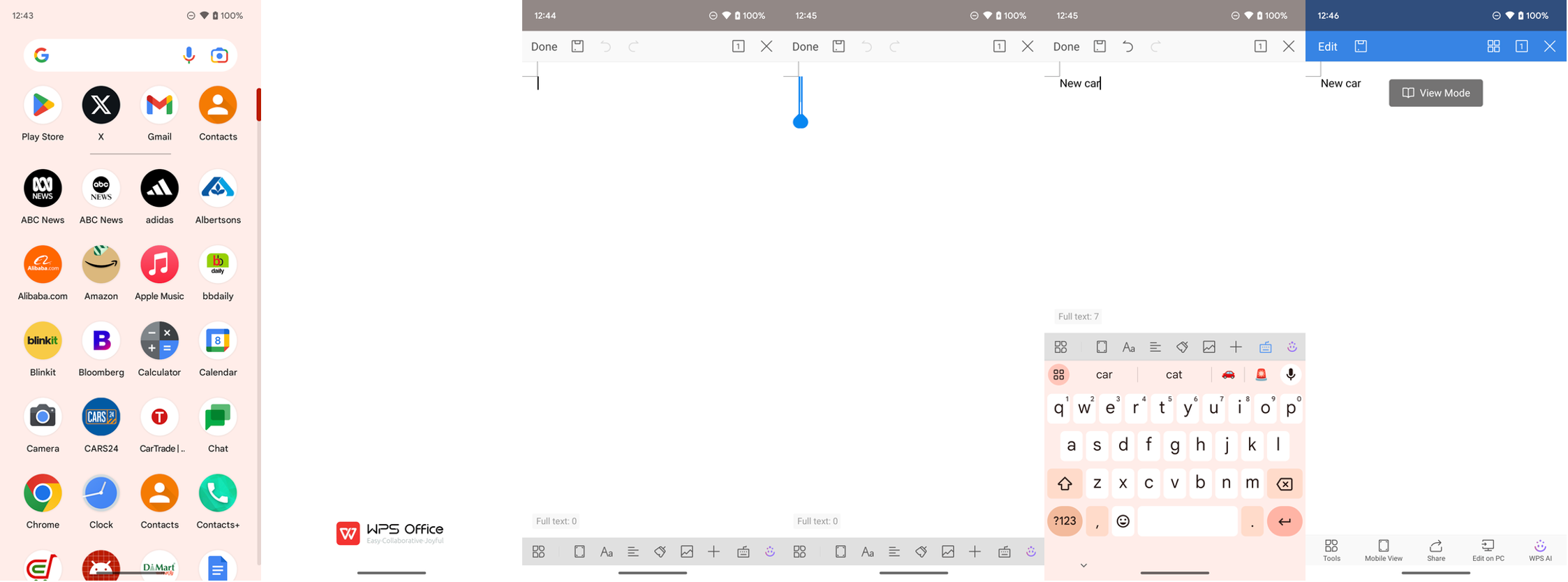} 
  \caption{\textbf{Trajectory Visualization: Document Creation on WPS Office.} 
  The agent initializes a new document: 
  (1) Launch the application via \texttt{Open(WPS Office)}; 
  (2) \texttt{Wait()} for initialization; 
  (3) Input the document title via \texttt{Input\_text(``New car'')}; 
  (4) \texttt{Wait()} for processing; 
  (5) Confirm the creation via \texttt{Click(0.079, 0.080)}; 
  (6) Terminate the task via \texttt{Finish()}.}
  \label{fig:task2_ex3}
\end{figure*}

\begin{figure*}[htbp]
  \centering
  \includegraphics[height=5cm, keepaspectratio]{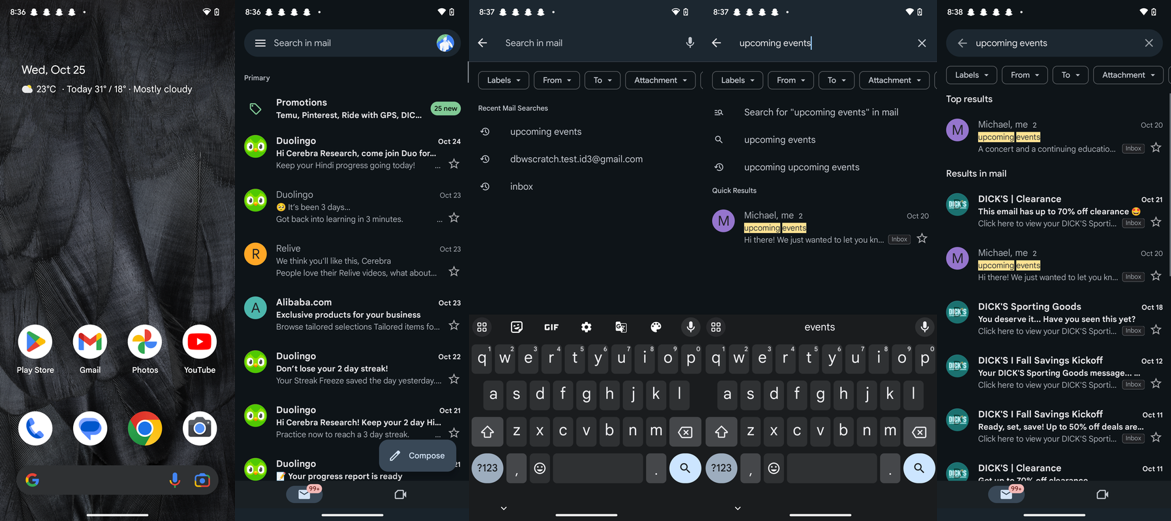} 
  \caption{\textbf{Trajectory Visualization: Email Search on Gmail.} 
  The agent filters inbox contents for specific events: 
  (1) Launch the application via \texttt{Open(Gmail)}; 
  (2) Focus on the search field via \texttt{Click(0.287, 0.084)}; 
  (3) Enter the query via \texttt{Input\_text(``upcoming events'')}; 
  (4) Execute the search via \texttt{Click(0.920, 0.904)}; 
  (5) Terminate the task via \texttt{Finish()}.}
  \label{fig:task3_ex1}
\end{figure*}

\begin{figure*}[htbp]
  \centering
  \includegraphics[height=5cm, keepaspectratio]{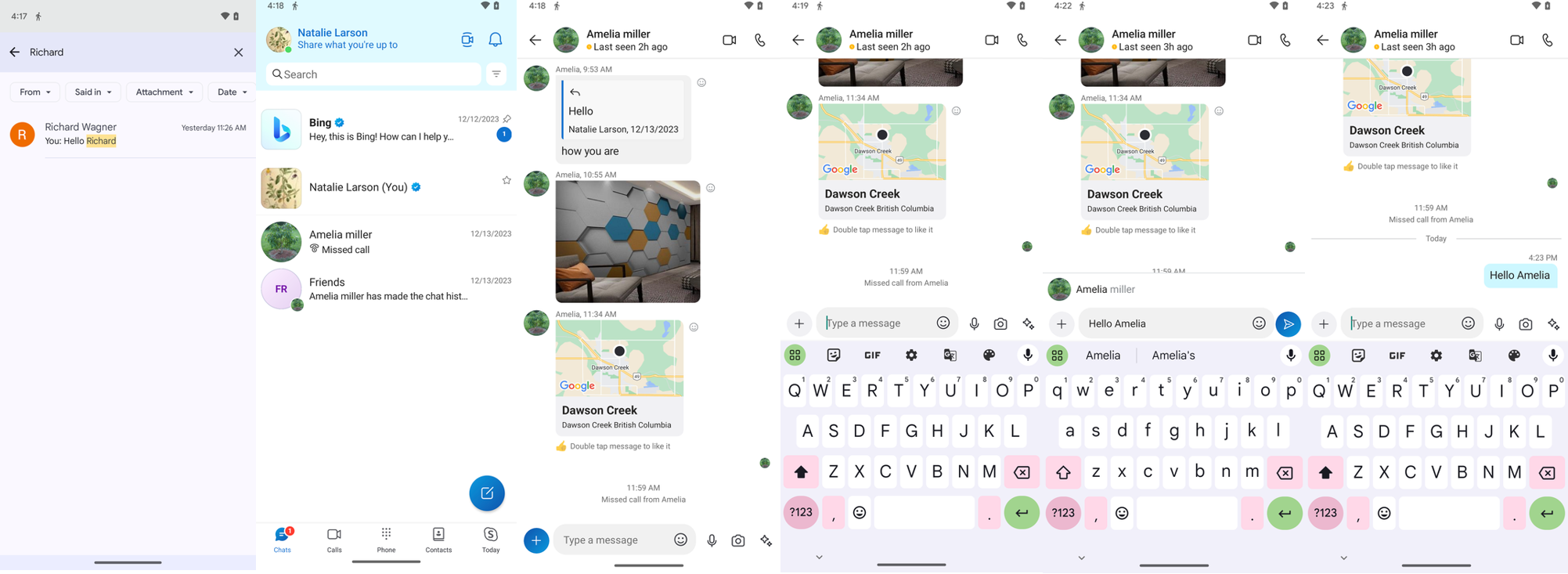} 
  \caption{\textbf{Trajectory Visualization: Messaging on Skype.} 
  The agent sends a greeting to a specific contact: 
  (1) Launch the application via \texttt{Open(Skype)}; 
  (2) Select the contact via \texttt{Click(0.500, 0.435)}; 
  (3) Activate the text input area via \texttt{Click(0.378, 0.942)}; 
  (4) Type the message via \texttt{Input\_text(``Hello Amelia'')}; 
  (5) Send the message via \texttt{Click(0.934, 0.574)}; 
  (6) Terminate the task via \texttt{Finish()}.}
  \label{fig:task3_ex2}
\end{figure*}

\begin{figure*}[htbp]
  \centering
  \includegraphics[height=5cm, keepaspectratio]{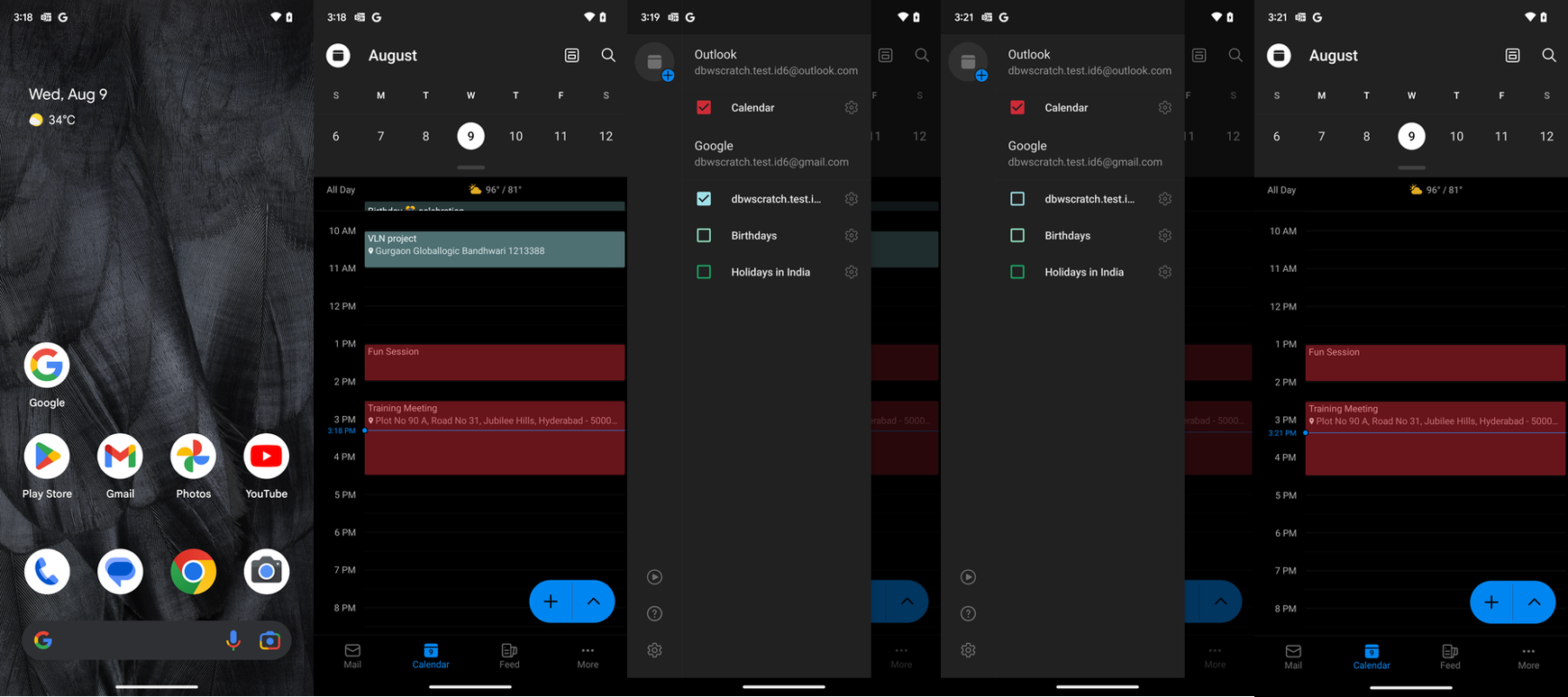} 
  \caption{\textbf{Trajectory Visualization: Account Management on Outlook.} 
  The agent navigates settings to manage accounts: 
  (1) Launch the application via \texttt{Open(Outlook)}; 
  (2) Open the side menu via \texttt{Click(0.078, 0.080)}; 
  (3) Select the target account via \texttt{Click(0.244, 0.285)}; 
  (4) Confirm the action via \texttt{Click(0.862, 0.387)}; 
  (5) Terminate the task via \texttt{Finish()}.}
  \label{fig:task3_ex3}
\end{figure*}

\begin{figure*}[htbp]
  \centering
  \includegraphics[height=5cm, keepaspectratio]{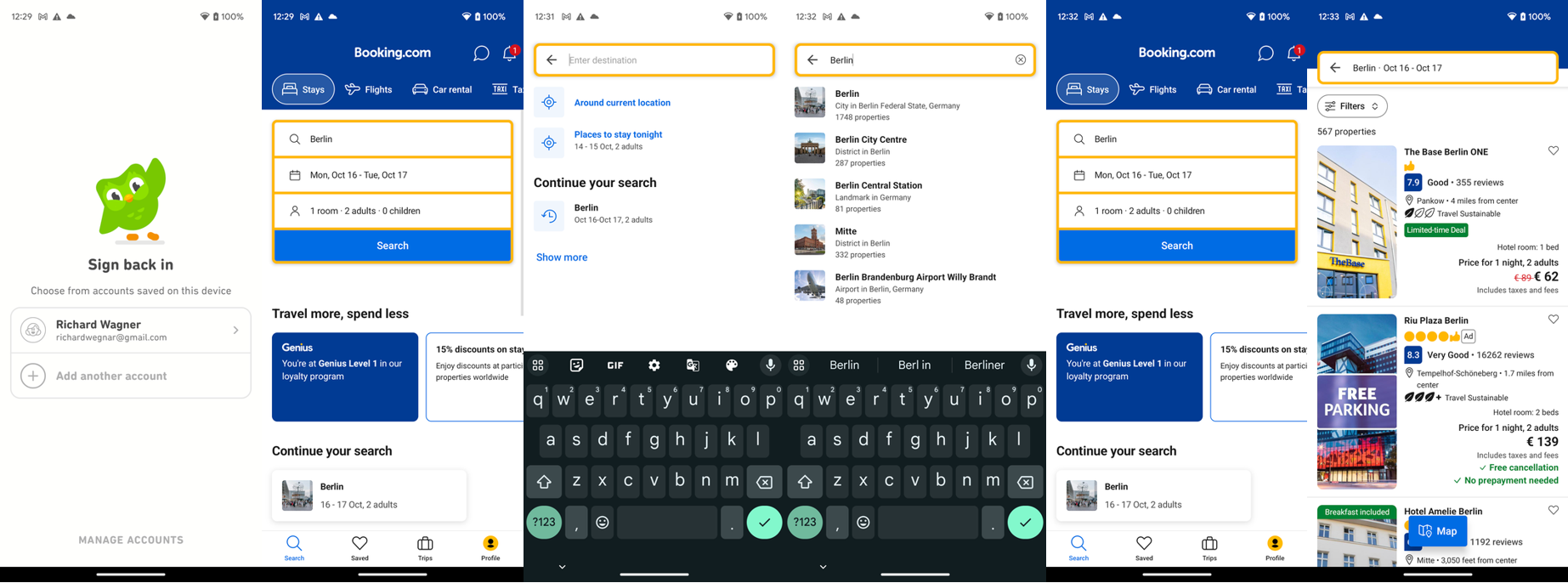} 
  \caption{\textbf{Trajectory Visualization: Hotel Search on Booking.com.} 
  The agent searches for accommodation in a specific location: 
  (1) Launch the application via \texttt{Open(Booking.com)}; 
  (2) Activate search via \texttt{Click(0.500, 0.239)}; 
  (3) Input the destination via \texttt{Input\_text(``Berlin'')}; 
  (4) Select the suggestion via \texttt{Click(0.578, 0.181)}; 
  (5) View results via \texttt{Click(0.500, 0.422)}; 
  (6) Terminate the task via \texttt{Finish()}.}
  \label{fig:task4_ex1}
\end{figure*}

\begin{figure*}[htbp]
  \centering
  \includegraphics[height=5cm, keepaspectratio]{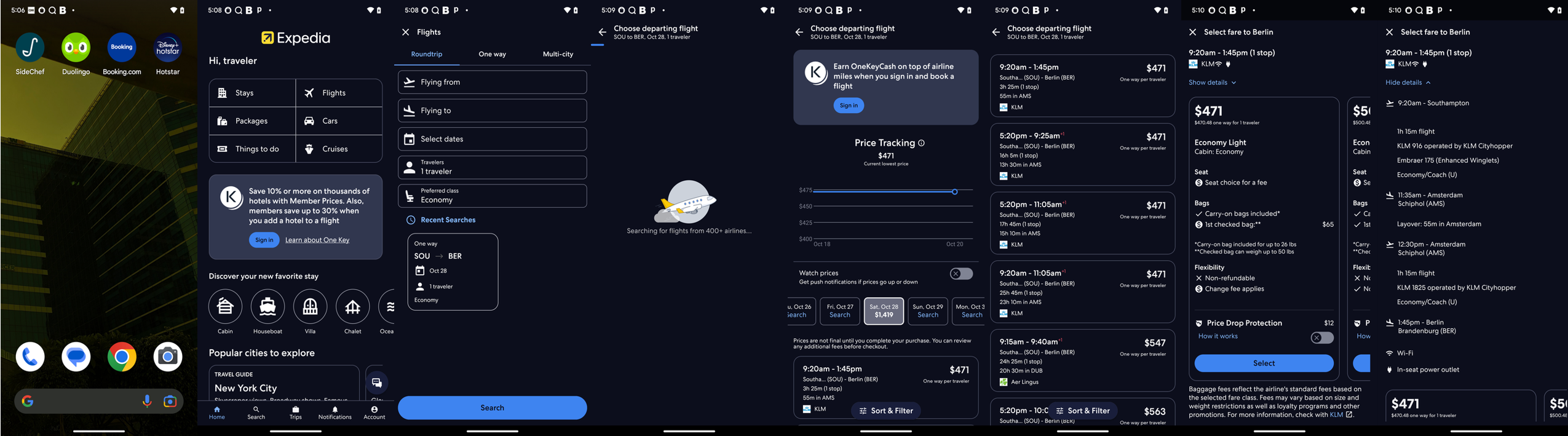} 
  \caption{\textbf{Trajectory Visualization: Flight Details on Expedia.} 
  The agent retrieves specific flight information: 
  (1) Launch the application via \texttt{Open(Expedia)}; 
  (2) Select flights via \texttt{Click(0.808, 0.203)}; 
  (3) Access recent searches via \texttt{Click(0.309, 0.586)}; 
  (4) \texttt{Wait()} for data retrieval; 
  (5) Browse results via \texttt{Scroll(down)}; 
  (6) Select the airline via \texttt{Click(0.500, 0.196)}; 
  (7) View detailed itinerary via \texttt{Click(0.137, 0.188)}; 
  (8) Terminate the task via \texttt{Finish()}.}
  \label{fig:task4_ex2}
\end{figure*}

\begin{figure*}[htbp]
  \centering
  \includegraphics[height=5cm, keepaspectratio]{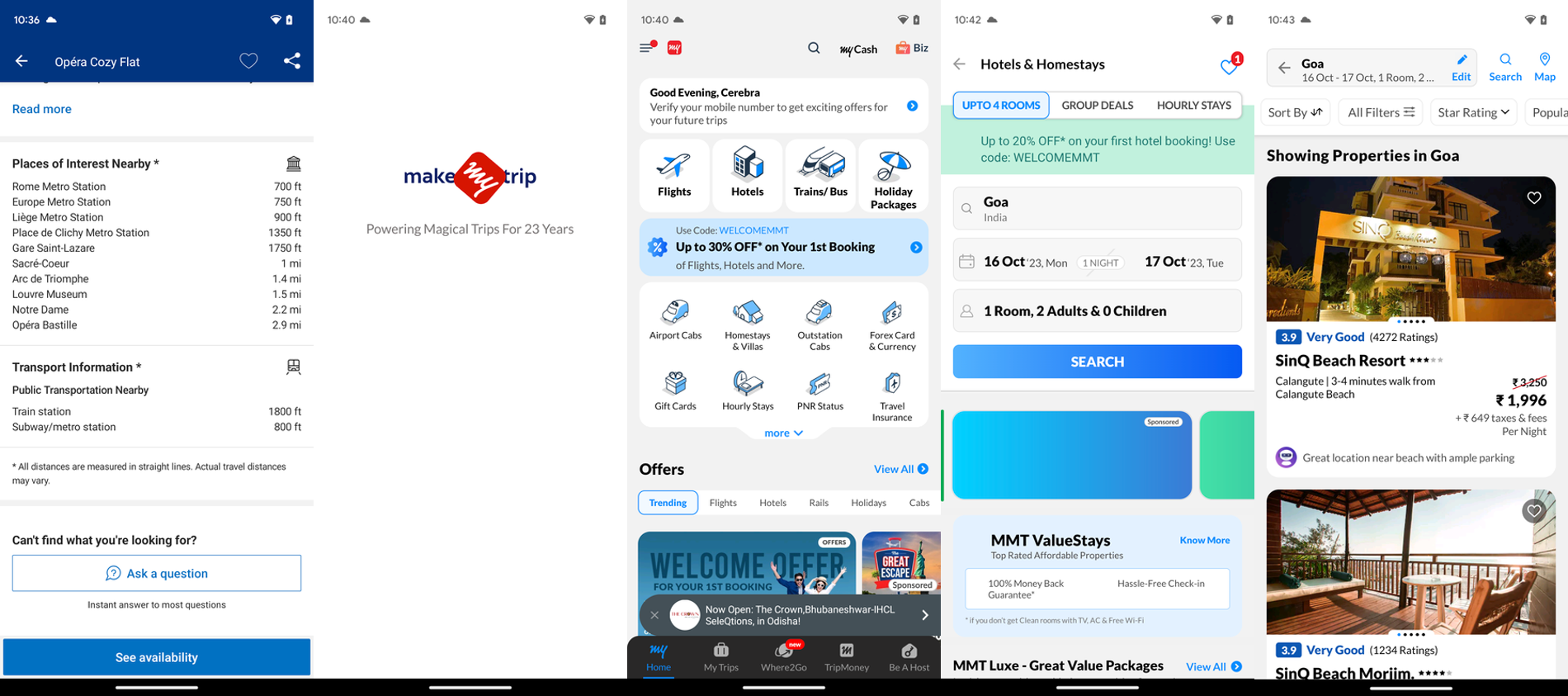} 
  \caption{\textbf{Trajectory Visualization: Hotel Inquiry on MakeMyTrip.} 
  The agent initiates a hotel search process: 
  (1) Launch the application via \texttt{Open(MakeMyTrip)}; 
  (2) \texttt{Wait()} for the interface to load; 
  (3) Select the hotel category via \texttt{Click(0.383, 0.252)}; 
  (4) Initiate the search via \texttt{Click(0.500, 0.518)}; 
  (5) Terminate the task via \texttt{Finish()}.}
  \label{fig:task4_ex3}
\end{figure*}

\begin{figure*}[htbp]
  \centering
  \includegraphics[height=5.5cm, keepaspectratio]{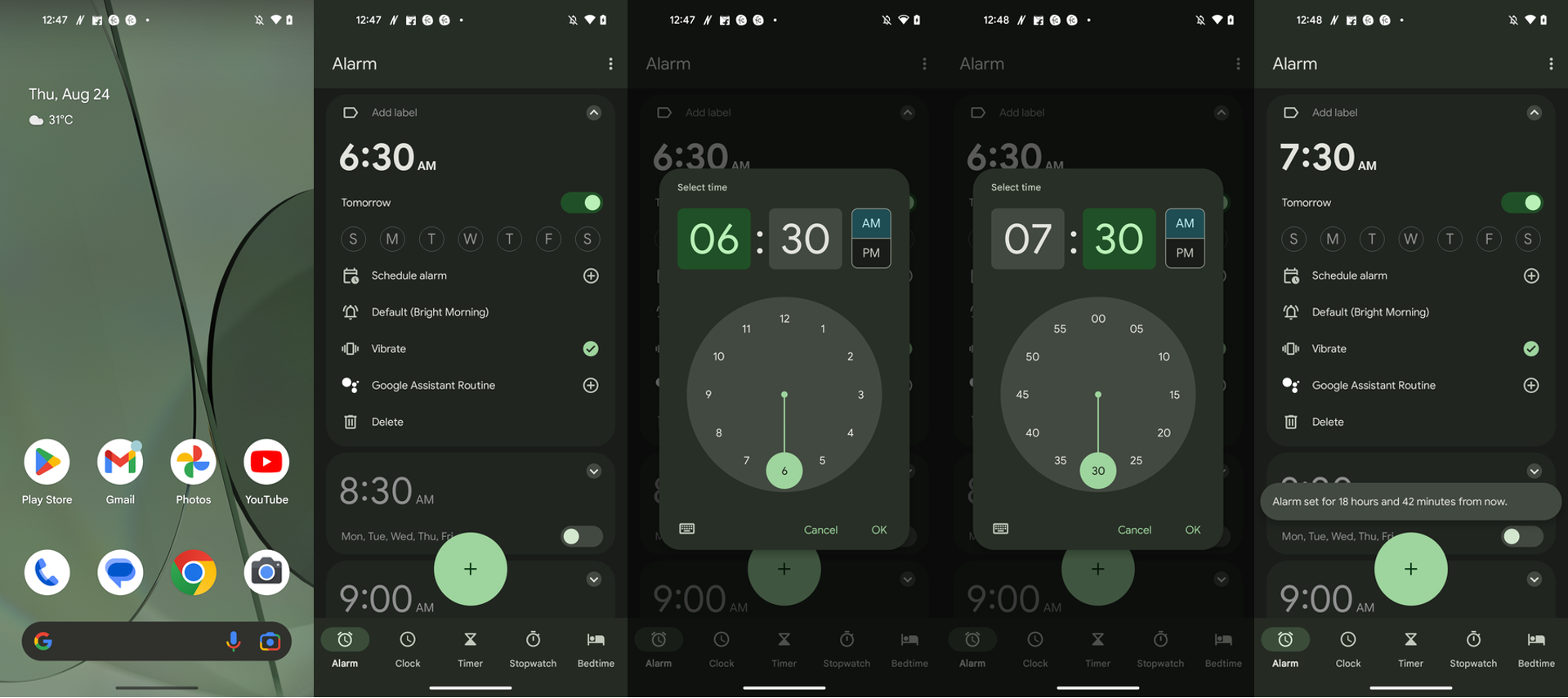} 
  \caption{\textbf{Trajectory Visualization: Alarm Adjustment on Clock.} 
  The agent modifies an existing alarm setting: 
  (1) Launch the application via \texttt{Open(Clock)}; 
  (2) Select the target alarm via \texttt{Click(0.233, 0.226)}; 
  (3) Adjust the hour setting via \texttt{Click(0.379, 0.660)}; 
  (4) Confirm the modification via \texttt{Click(0.801, 0.760)}; 
  (5) Terminate the task via \texttt{Finish()}.}
  \label{fig:task5_ex1}
\end{figure*}

\begin{figure*}[htbp]
  \centering
  \includegraphics[height=5.5cm, keepaspectratio]{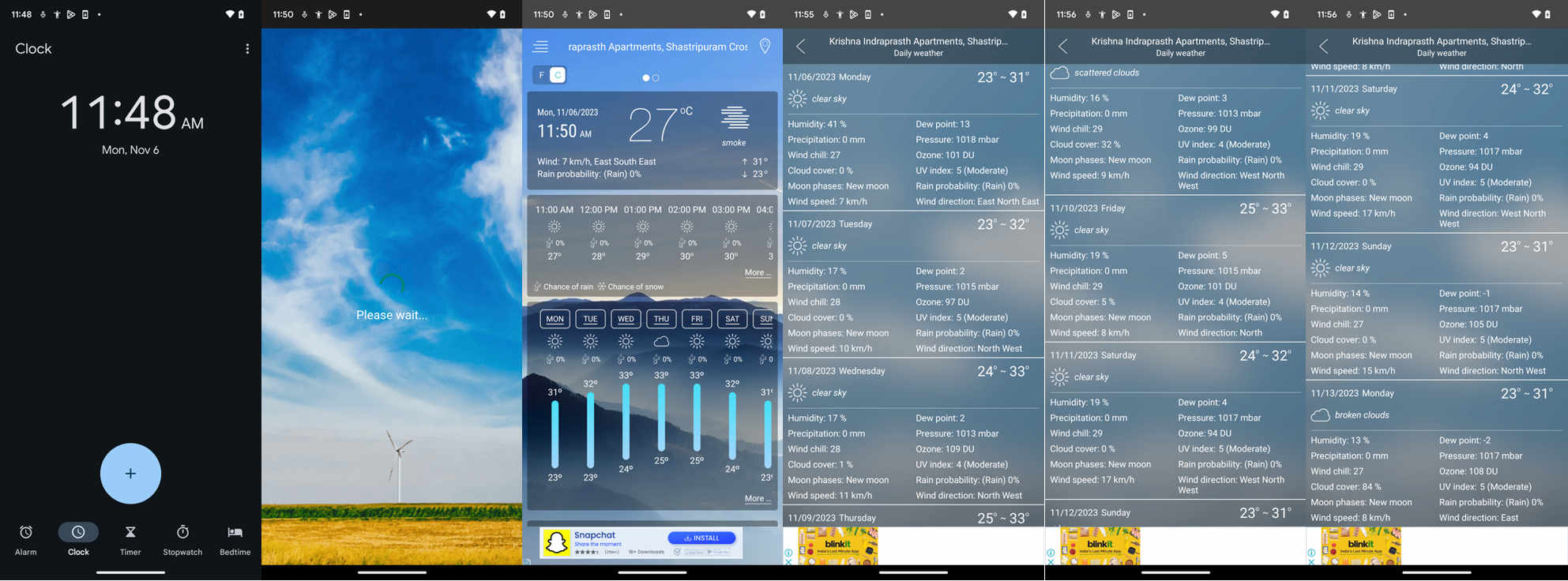} 
  \caption{\textbf{Trajectory Visualization: Forecast Retrieval on Weather.} 
  The agent navigates to the extended forecast view: 
  (1) Launch the application via \texttt{Open(Weather)}; 
  (2) \texttt{Wait()} for data synchronization; 
  (3) Expand the forecast view via \texttt{Click(0.905, 0.858)}; 
  (4) Browse the weekly data via \texttt{Scroll(down)}; 
  (5) Continue browsing via \texttt{Scroll(down)}; 
  (6) Terminate the task via \texttt{Finish()}.}
  \label{fig:task5_ex2}
\end{figure*}

\begin{figure*}[htbp]
  \centering
  \includegraphics[height=5.5cm, keepaspectratio]{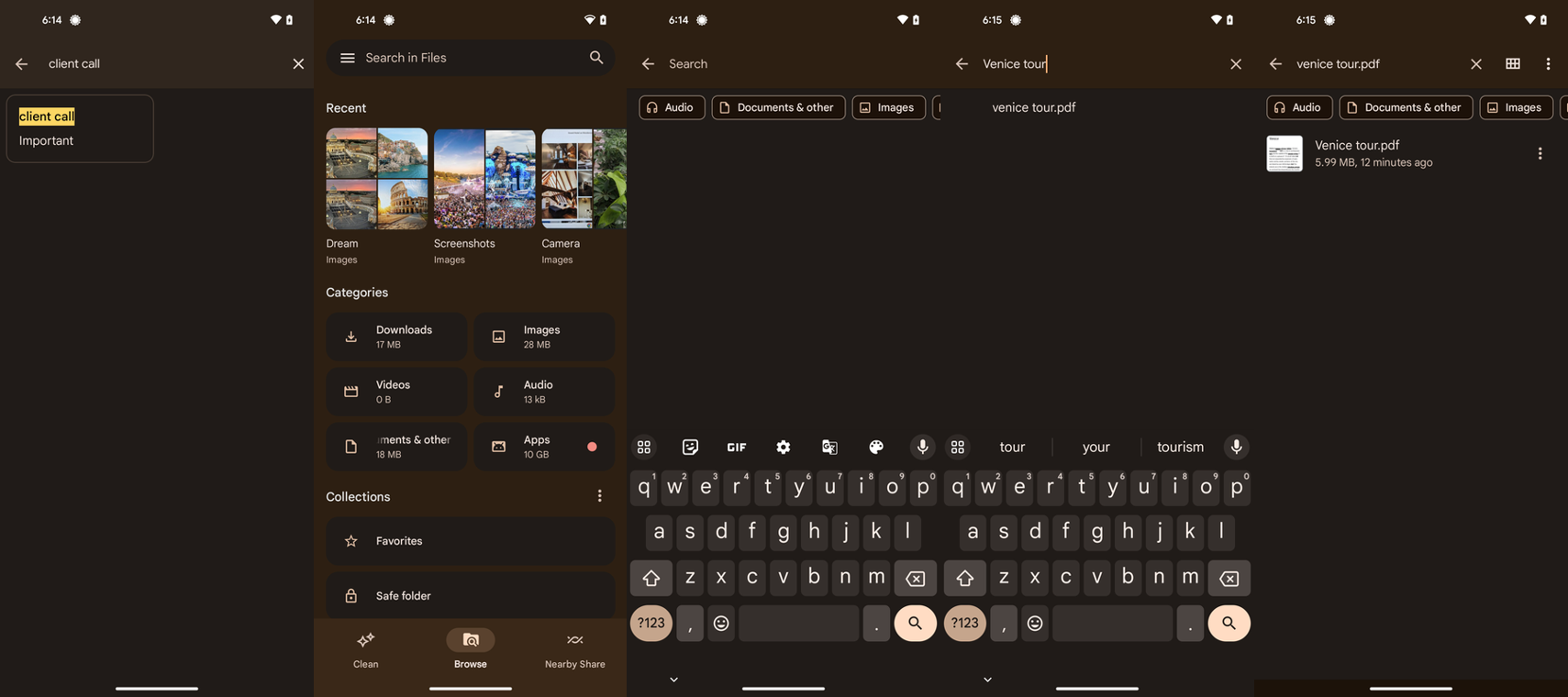} 
  \caption{\textbf{Trajectory Visualization: Document Retrieval on Google Files.} 
  The agent locates a specific file: 
  (1) Launch the application via \texttt{Open(files)}; 
  (2) Activate the search bar via \texttt{Click(0.902, 0.083)}; 
  (3) Input the filename via \texttt{Input\_text(``Venice tour'')}; 
  (4) Open the file via \texttt{Click(0.524, 0.154)}; 
  (5) Terminate the task via \texttt{Finish()}.}
  \label{fig:task5_ex3}
\end{figure*}

\begin{figure*}[htbp]
  \centering
  \includegraphics[height=5.5cm, keepaspectratio]{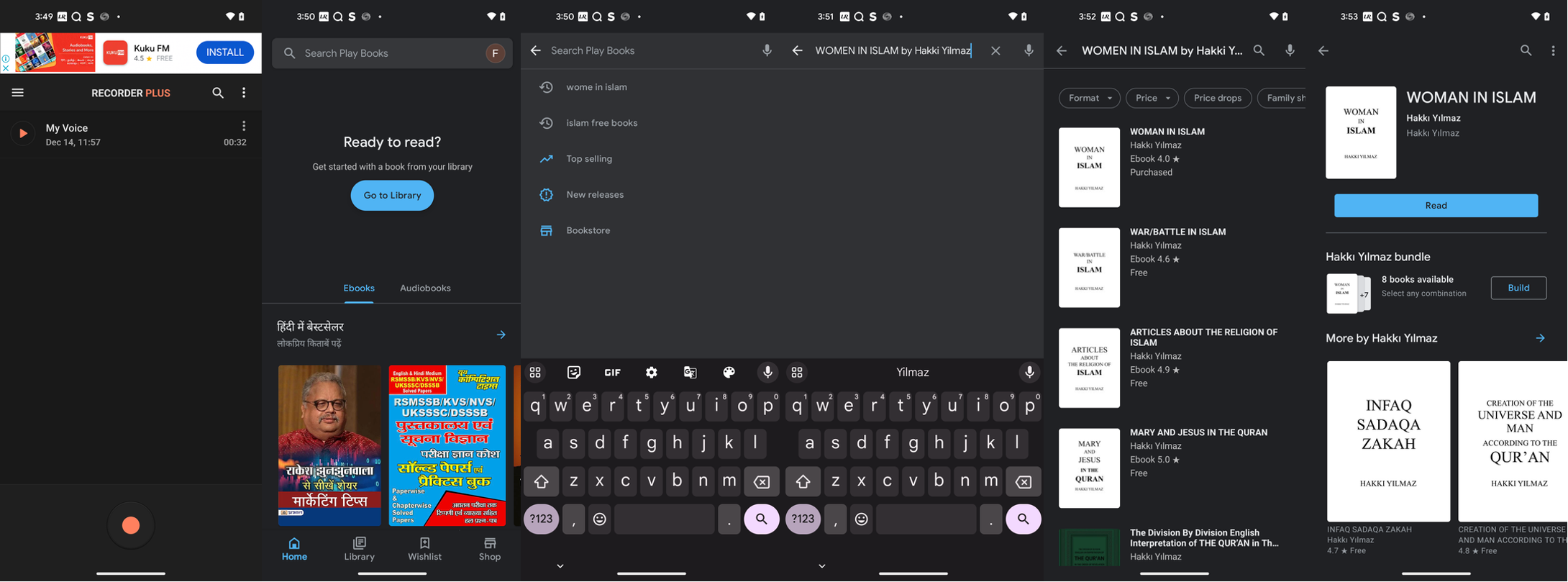} 
  \caption{\textbf{Trajectory Visualization: Book Search on Google Play Books.} 
  The agent searches for and selects a specific title: 
  (1) Launch the application via \texttt{Open(google play books)}; 
  (2) Focus on the search bar via \texttt{Click(0.324, 0.091)}; 
  (3) Enter the book title via \texttt{Input\_text(``WOMEN IN ISLAM by...'')}; 
  (4) Execute search via \texttt{Click(0.920, 0.899)}; 
  (5) Select the first result via \texttt{Click(0.500, 0.288)}; 
  (6) Terminate the task via \texttt{Finish()}.}
  \label{fig:task6_ex1}
\end{figure*}

\begin{figure*}[htbp]
  \centering
  \includegraphics[height=5.5cm, keepaspectratio]{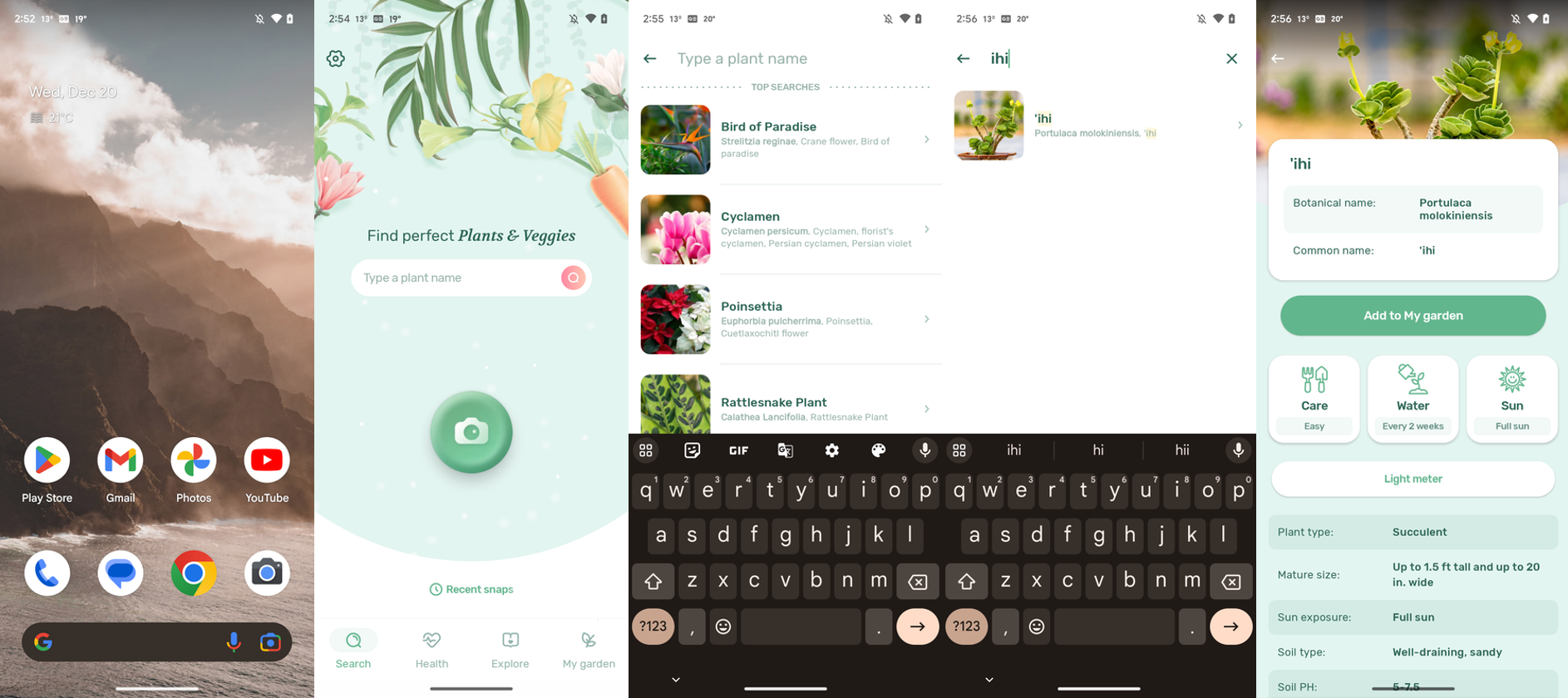} 
  \caption{\textbf{Trajectory Visualization: Plant Identification on Blossom.} 
  The agent searches for botanical information: 
  (1) Launch the application via \texttt{Open(Blossom)}; 
  (2) Activate the search bar via \texttt{Click(0.500, 0.398)}; 
  (3) Input the plant name via \texttt{Input\_text(``ihi'')}; 
  (4) Select the entry via \texttt{Click(0.225, 0.177)}; 
  (5) Terminate the task via \texttt{Finish()}.}
  \label{fig:task6_ex2}
\end{figure*}

\begin{figure*}[htbp]
  \centering
  \includegraphics[height=5.5cm, keepaspectratio]{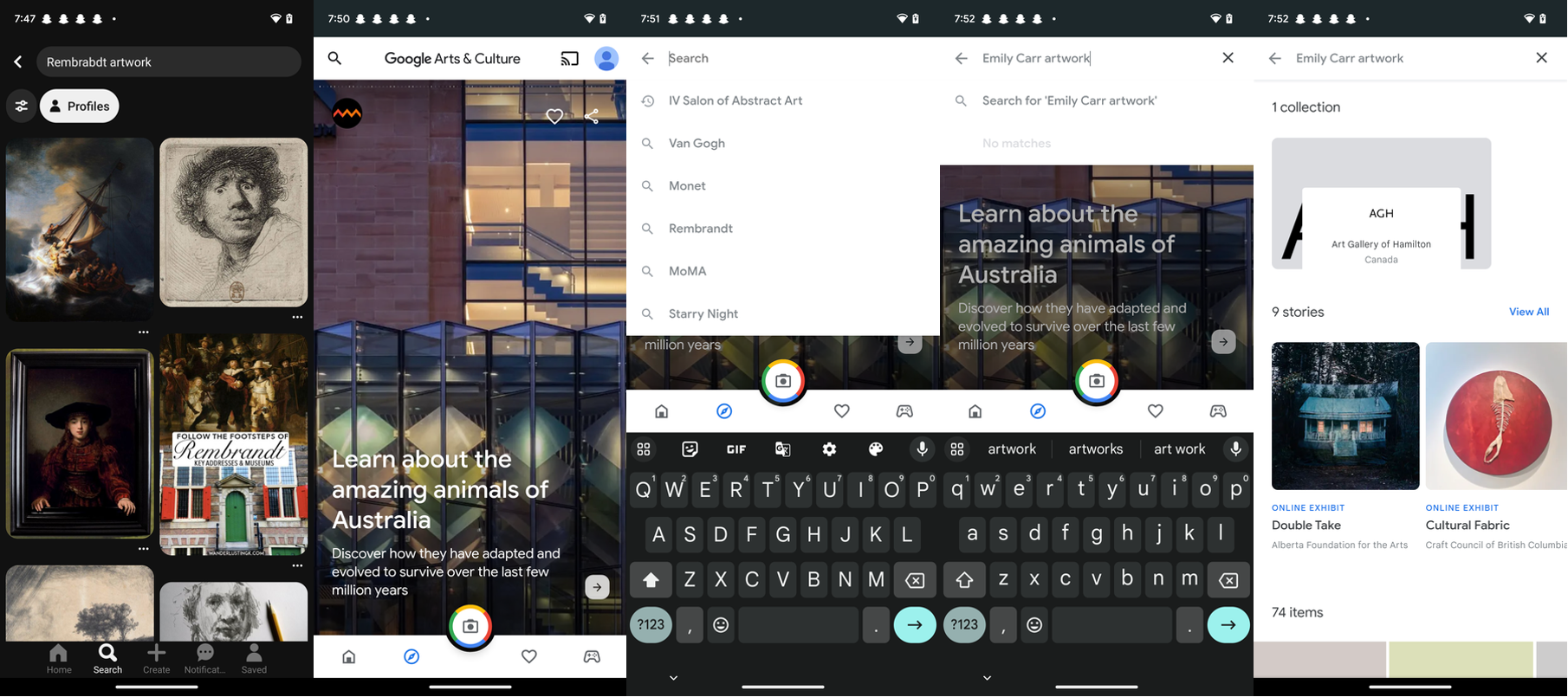} 
  \caption{\textbf{Trajectory Visualization: Artwork Search on Arts \& Culture.} 
  The agent retrieves specific artist works: 
  (1) Launch the application via \texttt{Open(Arts \& Culture)}; 
  (2) Click the search icon via \texttt{Click(0.067, 0.083)}; 
  (3) Input the query via \texttt{Input\_text(``Emily Carr artwork'')}; 
  (4) Confirm the search via \texttt{Click(0.920, 0.904)}; 
  (5) Terminate the task via \texttt{Finish()}.}
  \label{fig:task6_ex3}
\end{figure*}

\begin{figure*}[htbp]
  \centering
  \includegraphics[height=5.5cm, keepaspectratio]{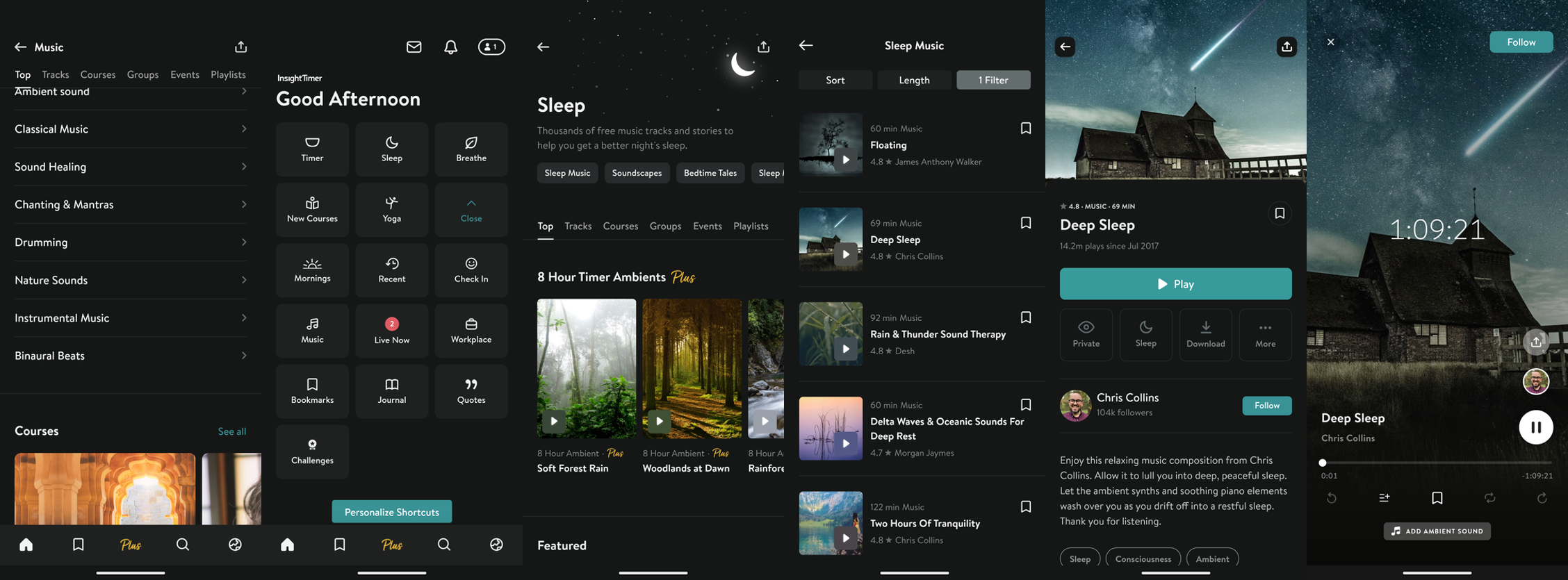} 
  \caption{\textbf{Trajectory Visualization: Audio Playback on Insight Timer.} 
  The agent navigates to and plays specific sleep content: 
  (1) Navigate back via \texttt{Back()}; 
  (2) Open the Sleep tab via \texttt{Click(0.499, 0.258)}; 
  (3) Select the music category via \texttt{Click(0.172, 0.297)}; 
  (4) Select the track ``Deep sleep'' via \texttt{Click(0.500, 0.412)}; 
  (5) Activate playback via \texttt{Click(0.499, 0.489)}; 
  (6) Terminate the task via \texttt{Finish()}.}
  \label{fig:task7_ex1}
\end{figure*}

\begin{figure*}[htbp]
  \centering
  \includegraphics[height=5.5cm, keepaspectratio]{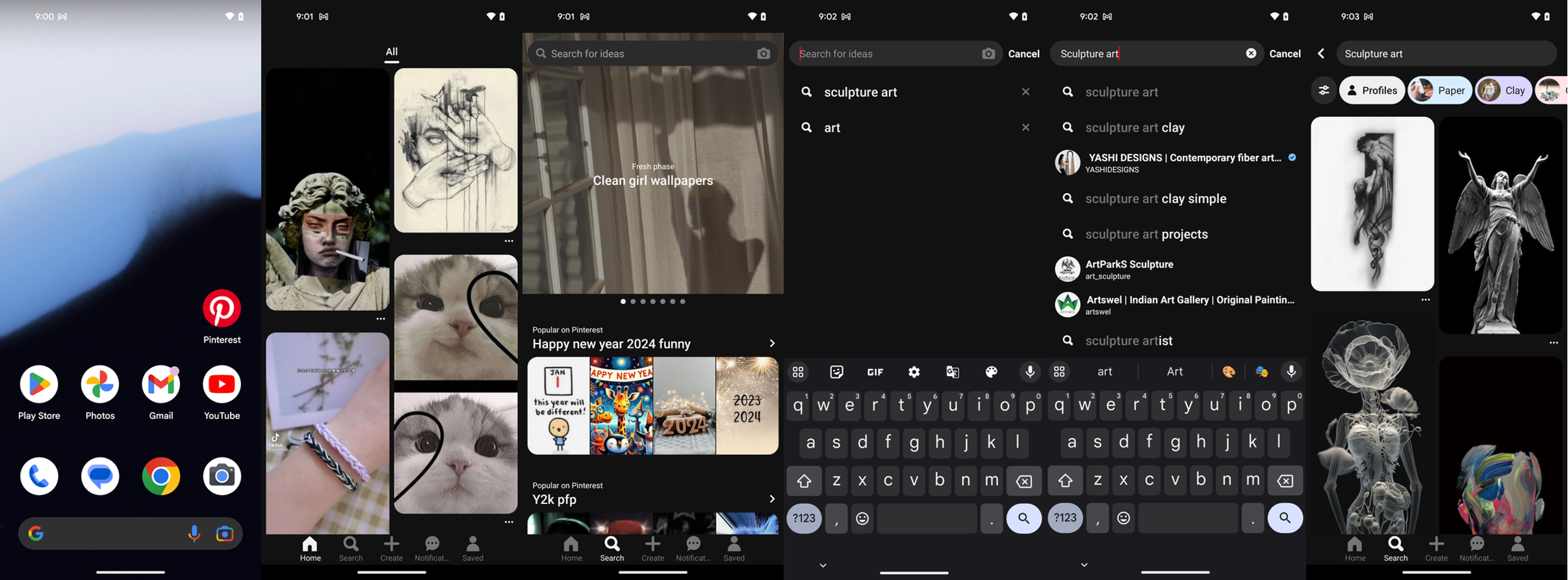} 
  \caption{\textbf{Trajectory Visualization: Content Discovery on Pinterest.} 
  The agent searches for a specific art category: 
  (1) Launch the application via \texttt{Open(Pinterest)}; 
  (2) Navigate to the search tab via \texttt{Click(0.344, 0.947)}; 
  (3) Activate the search bar via \texttt{Click(0.494, 0.092)}; 
  (4) Input the query via \texttt{Input\_text(``Sculpture art'')}; 
  (5) Execute search via \texttt{Click(0.920, 0.899)}; 
  (6) Terminate the task via \texttt{Finish()}.}
  \label{fig:task7_ex2}
\end{figure*}

\begin{figure*}[htbp]
  \centering
  \includegraphics[height=5.5cm, keepaspectratio]{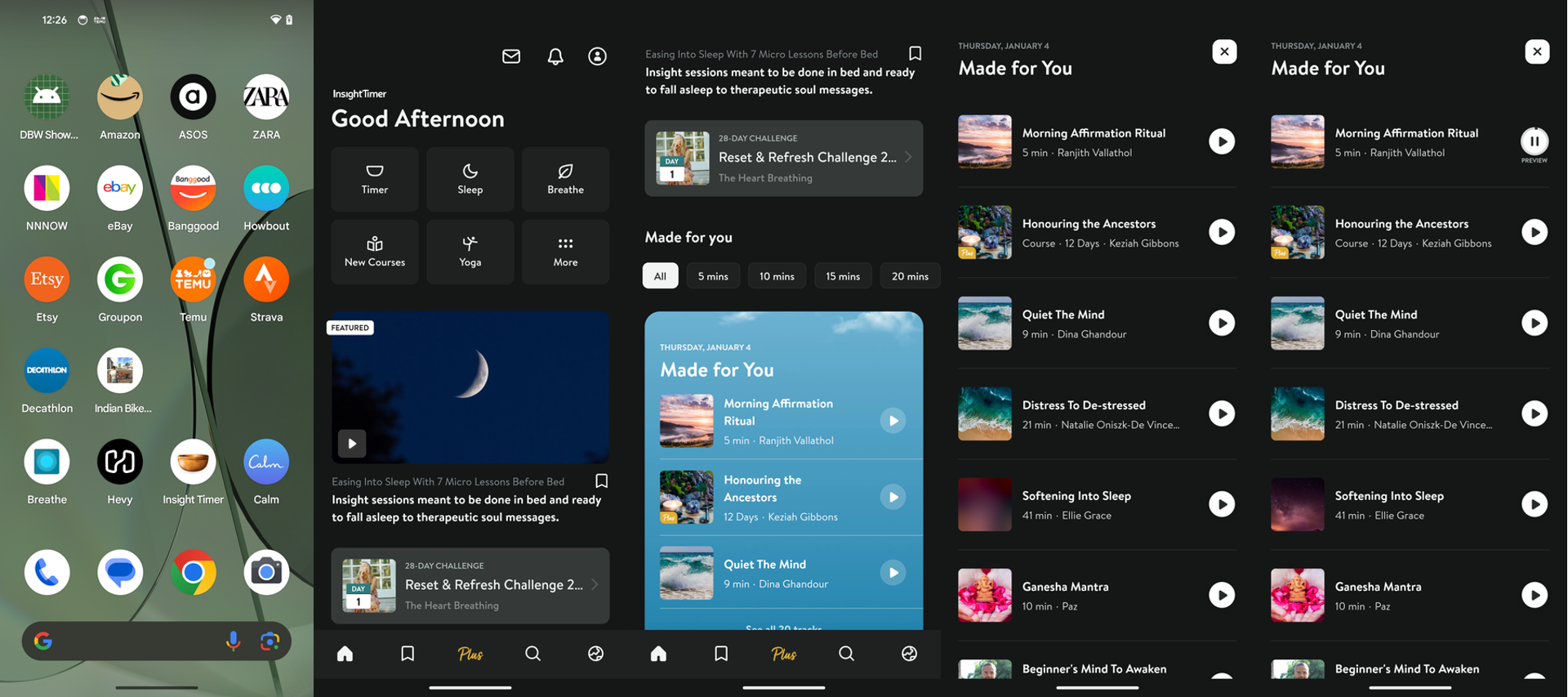} 
  \caption{\textbf{Trajectory Visualization: Guided Meditation on Insight Timer.} 
  The agent selects a morning ritual: 
  (1) Launch the application via \texttt{Open(Insight Timer)}; 
  (2) Browse content via \texttt{Scroll(down)}; 
  (3) Select the ``Morning Affirmation'' podcast via \texttt{Click(0.499, 0.603)}; 
  (4) Initiate playback via \texttt{Click(0.500, 0.203)}; 
  (5) Terminate the task via \texttt{Finish()}.}
  \label{fig:task7_ex3}
\end{figure*}

\end{document}